\documentclass{article}

\usepackage{microtype}
\usepackage{graphicx}
\usepackage{subfigure}
\usepackage{booktabs} 

\usepackage{hyperref}

\usepackage[accepted]{icml2024}

\usepackage{amsmath}
\usepackage{amssymb}
\usepackage{mathtools}
\usepackage{amsthm}

\usepackage{amsmath,amsfonts,bm, physics}

\def\eqref#1{Eq.~(\ref{#1})}

\def\1{\bm{1}}

\def\ru{{\textnormal{u}}}
\def\rv{{\textnormal{v}}}

\def\vmu{{\bm{\mu}}}

\def\va{{\bm{a}}}

\def\vj{{\bm{j}}}

\def\vt{{\bm{t}}}

\def\vv{{\bm{v}}}

\def\vx{{\bm{x}}}

\def\vz{{\bm{z}}}

\def\mA{{\bm{A}}}
\def\mB{{\bm{B}}}
\def\mC{{\bm{C}}}
\def\mD{{\bm{D}}}

\def\mG{{\bm{G}}}

\def\mI{{\bm{I}}}
\def\mJ{{\bm{J}}}
\def\mK{{\bm{K}}}

\def\mM{{\bm{M}}}

\def\mT{{\bm{T}}}

\def\mV{{\bm{V}}}
\def\mW{{\bm{W}}}
\def\mX{{\bm{X}}}

\DeclareMathAlphabet{\mathsfit}{\encodingdefault}{\sfdefault}{m}{sl}
\SetMathAlphabet{\mathsfit}{bold}{\encodingdefault}{\sfdefault}{bx}{n}

\def\gC{{\mathcal{C}}}

\def\gN{{\mathcal{N}}}

\def\*#1{\mathbf{#1}}
\def\$#1{\mathcal{#1}}
\def\^#1{\mathbb{#1}}

\newcommand{\E}{\mathbb{E}}

\newcommand{\R}{\mathbb{R}}

\usepackage{color}
\usepackage{tikz, pgfplots}
\usetikzlibrary{backgrounds}
\usetikzlibrary{plotmarks,spy}
\usepgfplotslibrary{groupplots}
\pgfplotsset{compat=newest}
\usepackage{multirow}
\usepackage[skins,xparse,breakable]{tcolorbox}

\usepackage[capitalize,noabbrev]{cleveref}

\theoremstyle{plain}
\newtheorem{theorem}{Theorem}[section]
\newtheorem{thm}[theorem]{Theorem}
\newtheorem{prop}[theorem]{Proposition}
\newtheorem{lem}[theorem]{Lemma}

\theoremstyle{definition}
\newtheorem{deft}[theorem]{Definition}
\newtheorem{assum}[theorem]{Assumption}
\newtheorem{exm}[theorem]{Example}
\theoremstyle{remark}
\newtheorem{rem}[theorem]{Remark}

\newcommand{\bsigma}{{\boldsymbol{\Sigma}}}
\newcommand{\beps}{{\boldsymbol{\varepsilon}}}

\newcommand{\btheta}{{\boldsymbol{\Theta}}}
\newcommand{\ttau}{{\widetilde\tau}}
\newcommand{\tmC}{{\widetilde\mC}}
\newcommand{\talpha}{{\widetilde\alpha}}
\newcommand{\tmG}{{\widetilde\mG}}
\newcommand{\bLambda}{{\boldsymbol{\Lambda}}}
\newcommand{\tbLambda}{{\widetilde{\boldsymbol{\Lambda}}}}

\definecolor{RED}{rgb}{0.7,0,0}
\definecolor{BLUE}{rgb}{0,0,0.69}
\definecolor{GREEN}{rgb}{0,0.6,0}
\definecolor{PURPLE}{rgb}{0.69,0,0.8}
\definecolor{ORANGE}{RGB}{255,103,0}
\definecolor{BROWN}{RGB}{100,20,45}

\newcommand{\RED}{\color[rgb]{0.70,0,0}}
\newcommand{\BLUE}{\color[rgb]{0,0,0.69}}
\newcommand{\GREEN}{\color[rgb]{0,0.6,0}}

\usepackage[textsize=tiny]{todonotes}

\icmltitlerunning{DEQs are Almost Equivalent to Not-so-deep Explicit Models for High-dimensional Gaussian Mixtures}

\begin{document}

\twocolumn[
\icmltitle{Deep Equilibrium Models are Almost Equivalent to Not-so-deep Explicit Models for High-dimensional Gaussian Mixtures}



\icmlsetsymbol{equal}{*}

\begin{icmlauthorlist}
\icmlauthor{Zenan Ling}{hust}
\icmlauthor{Longbo Li}{hust,ezhou}
\icmlauthor{Zhanbo Feng}{sjtu}
\icmlauthor{Yixuan Zhang}{HDU}
\icmlauthor{Feng Zhou}{Renmin}
\icmlauthor{Robert C. Qiu}{hust,ezhou}
\icmlauthor{Zhenyu Liao}{hust}
\end{icmlauthorlist}

\icmlaffiliation{hust}{School of Electronic Information and Communications, Huazhong University of Science and Technology, Wuhan, China.}
\icmlaffiliation{ezhou}{Ezhou Industrial Technology Research Institute, Huazhong University of Science and Technology, Wuhan, China.}
\icmlaffiliation{sjtu}{Department of CSE, Shanghai Jiao Tong University, Shanghai, China.}
\icmlaffiliation{HDU}{China-Austria Belt and Road Joint Laboratory on AI and AM, Hangzhou Dianzi University, Hangzhou, China.}
\icmlaffiliation{Renmin}{Center for Applied Statistics and School of Statistics, Renmin University of China, Beijing, China}
\icmlcorrespondingauthor{Zhenyu Liao}{zhenyu\_liao@hust.edu.cn}

\icmlkeywords{Random matrix theory, implicit neural networks, deep equilibrium models, high-dimensional statistics, over-parameterized neural neworks}
\vskip 0.3in
]



\printAffiliationsAndNotice{}  

\begin{abstract}
Deep equilibrium models (DEQs),  as typical implicit neural networks, have demonstrated remarkable success on various tasks. 
There is, however, a lack of theoretical understanding of the connections and differences between implicit DEQs and explicit  neural network models.
In this paper, leveraging recent advances in random matrix theory (RMT), we perform an in-depth analysis of the conjugate kernel (CK) and neural tangent kernel (NTK) matrices for implicit DEQs, when the input data are drawn from a high-dimensional Gaussian mixture.
We prove that, in this setting, the spectral behavior of these Implicit-CKs and NTKs depend on the DEQ activation function and initial weight variances, \emph{but only via a system of four nonlinear equations}. 
As a direct consequence of this theoretical result, we demonstrate that a shallow explicit network can be carefully designed to produce the same CK or NTK as a given DEQ. 
Despite derived here for Gaussian mixture data, empirical results show that the proposed theory and design principles also apply to popular real-world datasets.
\end{abstract}

\section{Introduction}

Recently, a novel approach in neural network (NN) design has gained prominence in the form of Implicit Neural Networks~\citep{NEURIPS2019_01386bd6,el2021implicit}. 
As typical implicit NNs, deep equilibrium models (DEQs) introduce a paradigm shift by resembling an infinite-depth weight-shared model with input-injection.
In contrast to traditional explicit NNs such as multi-layer perceptrons, recurrent neural networks, and residual networks, DEQs derive features by directly solving for fixed points.
These fixed points represent equilibrium states in the NN's computation, bypassing conventional layer-by-layer forward propagation. 

DEQs have demonstrated remarkable performance across a variety of applications, including computer vision~\citep{bai2020multiscale, xie2022optimization}, natural language processing~\citep{NEURIPS2019_01386bd6}, neural rendering~\citep{huang2021textrm}, and solving inverse problems~\citep{gilton2021deep}.
Despite the empirical success achieved by DEQs, our theoretical understanding of these implicit models is still limited. 
As a telling example, it remains unclear whether the training and/or generalization properties of implicit DEQs can be connected to those of explicit NN models.
\citet{NEURIPS2019_01386bd6} show that any deep explicit NN can be reformulated as an implicit DEQ with carefully-designed weight re-parameterization. 
Nonetheless, questions such as 
\begin{itemize}
    \item \emph{whether general DEQs have advantages over explicit networks}, or
    \item \emph{whether an equivalent explicit NN exists for a given implicit DEQ},
\end{itemize}
remain largely open.
Novel insights into these questions are strongly desired, since implicit DEQs incur significantly higher computational costs than explicit NNs during training and inference, as a consequence of their reliance on iterative solutions to fixed points~\citep{micaelli2023recurrence,fung2022jfb,ramzi2021shine,bai2021neural}.

In this paper, we provide \underline{\textbf{affirmative}} answers to the two open questions above, by considering input data following a Gaussian mixture model; refer to \Cref{rem:implicit_CK_VS_explicit_CK} in \Cref{sec:main} for a precise statement.
Building upon recent advances in random matrix theory (RMT), we investigate the high-dimensional behavior of DEQs by focusing on their conjugate kernel (CK) and neural tangent kernel (NTK) matrices.
These matrices offers an analytical assessment of the convergence and generalization properties for both implicit and explicit NNs, when the networks are wide, see~\citet{jacot2018neural} and \Cref{rem:NTK_and_CK} below for a detailed discussion.
For input data drawn a $K$-class Gaussian mixture model (GMM), we show, in the high-dimensional regime where the data dimension $p$ and their size $n$ are both large and comparable, 
that the Implicit-CKs and NTKs of DEQs can be evaluated via more accessible random matrix models that \emph{only} depend on the variance parameter and the activation function via a system of \emph{four} equations.
Possibly more surprisingly, these high-dimensional ``proxies'' of Implicit-CKs and NTKs have consistent forms with those of explicit NNs recently established  in~\citet{ali2021random,du2022lossless}.


Inspired by this observation, we establish the high-dimensional equivalence (in the sense of the CK and/or NTK) between implicit DEQs and explicit models, by matching their determining equations derived above.
In particular, our results reveal that a \textit{shallow} explicit NN with carefully designed activations is destined to exhibit \emph{identical} CK or NTK eigenspectral behavior as a given implicit DEQ, the depth of the latter being essentially \emph{infinite}. 
This implies, at least for GMM data, that an equivalent \emph{shallow} explicit NN (with the same amount of memory) can be designed, so as to avoid the significant computational overhead of implicit DEQs. 
Despite our theoretical results are derived for GMM data, we observe an unexpected close match between our theory and the empirical results on real-world datasets.

\subsection{Our Contributions}
Our contributions are summarized as follows. 
\begin{enumerate}
    \item[(1)] We provide, by considering high-dimensional GMM data, in Theorems~\ref{thm:imCK} and~\ref{thm:imNTK}, precise characterizations of CK and NTK matrices of implicit DEQs; we particularly show that the Implicit-CKs and NTKs \emph{only} depend on the DEQ variance parameter and activation function via a system of \emph{four} nonlinear equations.
    \item[(2)] We present, in \Cref{subsec:equivalence_implicit_explicit}, a comprehensive methodology for crafting ``equivalent'' shallow NNs that emulate a given implicit DEQ. 
    This involves determining the explicit NN activations through the system of equations derived in Theorems~\ref{thm:imCK} and~\ref{thm:imNTK}.
    We further illustrate the versatility of this framework in Examples~\ref{exm:Tanh} and~\ref{exm:ReLU}, showcasing its applications to widely-used ReLU and Tanh DEQs.
    \item[(3)]  We provide empirical evidence on GMM and real-world datasets such as MNIST, Fashion-MNIST, and CIFAR-10. 
    Our numerical results demonstrate that the carefully-designed explicit NNs exhibit performance on par with the given implicit DEQs. 
    This parity in performance is observed across both GMM and diverse realistic datasets, affirming the broad applicability and effectiveness of the proposed framework.
\end{enumerate}

\subsection{Related Works}

Here, we provide a brief review of related previous efforts.


\paragraph{Neural tangent kernels.}
Neural Tangent Kernel (NTK), initially proposed by~\citet{jacot2018neural}, examines the behavior of wide and deep NNs when trained using gradient descent with small steps.
Originally developed for fully-connected NNs, the NTK framework has since then been expanded to convolutional~\citep{arora2019exact}, graph~\citep{du2019graph}, and recurrent~\citep{alemohammad2020recurrent} settings. 
See also \Cref{rem:NTK_and_CK} below for the use of NTK in the analysis of DNNs.

\paragraph{Over-parameterized DEQs.}
\citet{feng2020neural} extend previous NTK studies to implicit DEQs and derive the exact expressions of the CK and NTK of ReLU DEQs. 
\citet{agarwala2022deep} investigate the NTK of DEQs under different random initializations.
These studies particularly asserts that (\romannumeral1) the Implicit-NTKs of DEQs are equivalent to the corresponding weight-untied models in the infinitely wide regime and (\romannumeral2) implicit DEQs have non-degenerate NTKs even in the infinite depth limit. 
These observations align with our findings. 
The connections between Implicit-CKs/NTKs and Explicit-CKs/NTKs, however, remain unexplored. 
Here we perform a fine-grained analysis on the Implicit-CKs and NTKs of DEQs and establish their equivalence to explicit NN model.
Also, while training dynamics (and global convergence) of over-parameterized DEQs have been investigated in previous works~\citep{2021A,gao2022,ling2023global,truong2023global} in the NTK regime, it remain unclear how these DEQ training dynamics distinguishes from those of explicit models. 

\paragraph{Random matrix theory and NNs.} 
Random matrix theory (RMT) has emerged as a versatile and potent tool for evaluating the behavior of large-scale systems characterized by a substantial ``degree of freedom." 
Its application has been increasingly embraced in the realm of NN analysis, spanning shallow~\citep{pennington2017nonlinea,liao2018spectrum,liao2018dynamics} and deep~\citep{benigni1904eigenvalue,fan2020spectra,pastur2022eigenvalue,pastur2023random} models, homogeneous (\emph{e.g.,} standard normal)~\citep{pennington2017nonlinea,mei2022generalization} and mixture-type datasets~\citep{liao2018spectrum,ali2021random,du2022lossless}. 
From a technical perspective, the most relevant papers are~\citet{ali2021random} and~\citet{du2022lossless}, in which the eigenspectra of CK and NTK were evaluated, for explicit single-hidden-layer NN in~\citet{ali2021random} and explicit deep NNs with multiple (but finite) layers in~\citet{du2022lossless}.
Here, we extend previous analysis to implicit DEQs with an effectively \emph{infinite} number of layers, and establish an equivalence between implicit and explicit NN models.

\section{Preliminaries}
\paragraph{Notations.}

We use $\gN(0,\mI)$ to denote standard multivariate Gaussian distribution.
For a vector $\vv$, $\|\vv\|$ is the Euclidean norm of $\vv$. 
For a matrix $\mA$, we use $\mA_{ij}$ to denote its $(i,j)$-th entry, $\left\|\mA\right\|_F$ to denote its Frobenius norm, and $\left\|\mA\right\|$ to denote its spectral norm. 
We use $\odot$ to denote the Hadamard product between matrices of the same size.
We let $\order{\cdot}$, $\Theta(\cdot)$ and $\Omega(\cdot)$ denote standard Big-O, Big-Theta, and Big-Omega notations, respectively.

In this paper, we focus on the DEQ model introduced in~\citet{NEURIPS2019_01386bd6}, defined as follows.
\begin{deft}[Deep equilibrium model, DEQ]\label{def:deq}
Let $\mX=[\vx_1,\cdots,\vx_n]\in \mathbb{R}^{p \times n}$ denote the input data, consider a vanilla DEQ with output $f(\vx_i)$ given by
\begin{equation}
    f(\vx_i)=\va^\top\vz_i^*,  
\end{equation}
where $\va\in \mathbb{R}^m$ and $\vz_i^{(*)} \triangleq \lim_{l \rightarrow \infty} \vz_i^{(l)}\in\R^{m}$ with
\begin{equation}
    \vz_i^{(l)} = \frac{1}{\sqrt{m}}\phi\left(\sigma_a\mA \vz_i^{(l-1)}+\sigma_b\mB \vx_i \right) \in \mathbb{R}^{m}, \, \text{for}\, l\geq 1, 
    \label{eq:deql}
\end{equation}
for some appropriate initialization $\vz_i^{(0)}$.
Here, $\mA \in \mathbb{R}^{m \times m}$  and  $\mB \in \mathbb{R}^{m \times p}$ are the DEQ weight parameters, $\sigma_a, \sigma_b\in \R$ are  constants, and $\phi$ is an element-wise activation.
Note that $\vz_i^*$ can also be determined as the equilibrium point of
\begin{equation}
   \vz_i^{*} = \frac{1}{\sqrt{m}}\phi\left(\sigma_a\mA \vz_i^{*}+\sigma_b\mB \vx_i \right). 
   \label{eq:deq}
\end{equation}    
\end{deft}

We position ourselves under the following assumptions on the weights and activation functions of the DEQ.
\begin{assum}[Weights initialization]
    For the DEQ model in \Cref{def:deq}, the weight parameters $\va\in\R^{m}$, $\mA\in\R^{m \times m}$ and $\mB\in\R^{m \times p}$ are initialized as independent random vector or matrices having \emph{i.i.d.} entries of zero mean, unit variance, and finite fourth-order moment. 
    \label{assum:initial}
\end{assum}
\begin{assum}[Activation function]
    For the DEQ model in \Cref{def:deq}, the activation function $\phi$ is $L_1$-Lipschitz, and at least four-times weakly differentiable with respect to standard normal measure, \emph{i.e.,} $\max_{k\in\{0,1,2,3,4\}}{|\E[\phi^{(k)}(\xi)]|} < \infty $ for $\xi\sim\gN(0,1)$. 
    \label{assum:activation}
\end{assum}
Here, we consider the weak differentiability of a function, which generalizes the notation of derivative for non-differentiable (but integrable) functions. 
Specifically, using the Gaussian integration by parts formula, one has $\E[\phi'(\xi)]=\E[\xi\phi(\xi)]$ for $\xi\sim \gN(0,1)$, as long as the right-hand side expectation exists. 
It can be checked that \Cref{assum:activation} holds for commonly-used smooth, \emph{e.g.}, Tanh, and piecewise linear activations, \emph{e.g.}, ReLU and Leaky ReLU.

For the stability of training and inference of DEQs, it is of crucial significance to guarantee the existence and uniqueness of the equilibrium point in \eqref{eq:deq}~\citep{winston2020monotone,el2021implicit}. 
To that end, we introduce the following assumption on the variance parameter $\sigma_a$.
\begin{assum}[Variance parameter]\label{cond:G*}
For the DEQ model in \Cref{def:deq}, the variance parameter $\sigma_a$ in \eqref{eq:deq} satisfies $\sigma_a^2< 1/(4L_1^2)$, with $L_1$ the Lipschitz constant of the activation function $\phi$ as demanded in \Cref{assum:activation}.
\end{assum}
It follows from \Cref{assum:initial} and standard singular value bounds of random matrices~\citep{vershynin2018high} that $\|\mA\|\leq2 \sqrt m$ with high probability.
Then, by noting that $\phi(\cdot)$ is $L_1$-Lipschitz, one has, under \Cref{cond:G*}, that the transformation in~\eqref{eq:deql} is a \emph{contractive} mapping. 
This thus ensures the existence of the unique fixed point $\vz^*$.

We are interested in the conjugate kernel and the neural tangent kernel (Implicit-CK and Implicit-NTK, for short) of the implicit DEQ in \Cref{def:deq}.

\begin{rem}[On CKs and NTKs]\label{rem:NTK_and_CK}
Conjugate kernels (CKs) and neural tangent kernels (NTKs) are closely related kernels useful in the analysis of NNs~\cite{fan2020spectra}.
During gradient descent training, the network parameters change and the NTK also evolves over time.
It has been shown in~\citet{jacot2018neural} and follow-up works that for sufficiently wide DNNs trained on gradient descent with small learning rate: 
(i) the NTK is approximately constant after initialization; and 
(ii) running gradient descent to update the network parameters is \emph{equivalent} to kernel gradient descent with the NTK.
This duality allows one to assess the training dynamics, generalization, and predictions of wide DNNs as \emph{closed-form} expressions involving NTK eigenvalues and eigenvectors, see~\citet[Section~6]{bartlett2021Deep}.
\end{rem}

For Implicit-CKs and NTKs, we recall the following result.
\begin{prop}[Implicit-CKs and NTKs of DEQ,~\citep{feng2020neural, gao2023wide}]\label{prop:Implicit-CK-NTK}
Under Assumptions~\ref{assum:initial}-\ref{cond:G*}, the Implicit-CK of the DEQ model in \Cref{def:deq} takes the following form:    
\begin{equation}\label{eq:imck}
    \mG^* = \lim_{l\rightarrow\infty}\mG^{(l)},
\end{equation}
where the $(i,j)$-th entry of $\mG^{(l)}$ is defined recursively as\footnote{ Note that the expectation is conditioned on the input data, and is taken with respect to the random weights. \label{footnote:conditional_expectation}} $\mG^{(l)}_{ij} = \E[(\vz_i^{(l)})^\top\vz_j^{(l)}]$, \emph{i.e.}, $\mG^{(0)}_{ij} = (\vz_i^{(0)})^\top\vz_j^{(0)}$ and 
\begin{equation}
     \mG^{(l)}_{ij} =\E_{(\ru_l,\rv_l)}[\phi(\ru_l)\phi(\rv_l)], 
\end{equation}
with $(\ru_l,\rv_l)\sim \gN\left(0, \begin{bmatrix}
	 \boldsymbol{\Lambda}_{ii}^{(l)} &  \boldsymbol{\Lambda}_{ij}^{(l)}
	\\
	\boldsymbol{\Lambda}_{ji}^{(l)} &  \boldsymbol{\Lambda}_{jj}^{(l)}
\end{bmatrix}\right)$ and $\boldsymbol{\Lambda}_{ij}^{(l)}=\sigma_a^2\mG^{(l-1)}_{ij} + \sigma_b^2\vx_i^\top\vx_j$, for $l \geq 1$.
The corresponding Implicit-NTK takes the form $\mK^*=\lim_{l\rightarrow\infty}\mK^{(l)}$ whose the $(i,j)$-th entry  is defined as
 \begin{equation}
     \mK^{(l)}_{ij} = \sum_{h=1}^{l+1}\left(\mG^{(h-1)}_{ij}\prod_{h'=h}^{l+1}\dot\mG^{(h')}_{ij}\right),
 \end{equation}
with $\dot\mG^{(l)}_{ij} = \sigma_a^2\E_{(\ru_l,\rv_l)\sim \gN(0,\boldsymbol{\Lambda}_{ij}^{(l)})}[\phi'(\ru_l)\phi'(\rv_l)]$ so that
\begin{equation}
        \mK^*_{ij} \equiv \mG^*_{ij}/(1-\dot\mG^*_{ij}).
        \label{eq:imntk}
\end{equation}
\end{prop}
The existence and uniqueness of the DEQ Implicit-CK~and~NTK expressions given in \Cref{prop:Implicit-CK-NTK} are guaranteed under Assumptions~\ref{assum:initial}-\ref{cond:G*}, see \citet{gao2023wide} for a detailed proof using a Gaussian process argument.

For the purpose of our theoretical analysis, we consider input data drawn from the following high-dimensional Gaussian mixture model.
\begin{assum}[High-dimensional Gaussian mixture model, GMM]\label{def:gmm}
Consider $n$ data vectors $\vx_1,\cdots,\vx_n\in\R^p$ independently drawn from one of the $K$-class Gaussian mixtures $\mathcal{C}_1,\cdots,\mathcal{C}_K$, \emph{i.e.}, for $\vx_i\in\mathcal{C}_a$, we have
\begin{equation}
\sqrt p\vx_i\sim\gN(\vmu_a,\mC_a), \quad a\in\{1,\cdots,K\}.
\label{eq:GMM}
\end{equation}
We assume, for $n, p$ both large that (\romannumeral1) $p = \Theta(n)$ and $n_a$ the cardinality of class $\mathcal{C}_a$ satisfies $n_a = \Theta(n)$; (\romannumeral2) $\|{\vmu}_a\|=\order{1}$; (\romannumeral3) for $\mC^{\circ}\equiv\sum_{a=1}^{K}\frac{n_a}{n}\mC_a$ and $\mC_a^\circ \equiv \mC_a-\mC^{\circ}$, we have $\|\mC_a\|=\order{1}$, $\tr \mC_a^\circ=\order{\sqrt{p}}$ and $\tr(\mC_a\mC_b)=\order{p}$ for $a,b\in \{1,\cdots,K\}$. 
\end{assum}


\begin{rem}[On GMM in \Cref{def:gmm}]\normalfont
The normalization by $\sqrt p$ of the GMM in \eqref{eq:GMM} is commonly used in the literature of high-dimensional statistics and over-parameterized DNNs and ensures that the data vectors have bounded Euclidean norms $\| \vx_i \| = \order{1}$ with high probability for $n,p$ large.
The assumptions on the scaling of the means and covariances in \Cref{def:gmm}, despite being technical at first sight, ensure the GMM classification in \eqref{eq:GMM} remains non-trivial for $n,p$ large, and have been extensively studied in the literature for various ML methods ranging from LDA, spectral clustering, SVM, to shallow and deep neural networks, see for example \citep{couillet2016kernel,louart2018random,dobriban2018high,liao2019large,elkhalil2020large,couillet2022RMT4ML,du2022lossless} as well as \citep[Section~2]{blum2020foundations}.

\end{rem}

\section{Main Results}
\label{sec:main}
In this section, we present in \Cref{subsec:high_characterization_of_CK_and_NTK} our main technical results on the high-dimensional characterization of CK and NTK matrices of implicit DEQs, in Theorems~\ref{thm:imCK}~and~\ref{thm:imNTK}, respectively.
We then show in \Cref{subsec:equivalence_implicit_explicit} that the proposed theoretical analysis allows to construct, for a given implicit DEQ model, an equivalent and shallow \emph{explicit} NN model that shares the same CK eigenspectra as the implicit DEQ.



\subsection{High-dimensional Characterization of Implicit-CK~and~NTK Matrices}
\label{subsec:high_characterization_of_CK_and_NTK}

Let us start by introducing some notations that will be used in the remainder of this paper.
For GMM data in \Cref{def:gmm}, denote 
\begin{equation*}
    \mJ\equiv[\vj_1,\cdots,\vj_K] \in\R^{n\times K}, \quad \vj_a \in \R^n,
\end{equation*}
with $[\vj_a]_i = 1_{\vx_i\in\gC_a}$ of class $\mathcal C_a$, $a \in \{ 1, \ldots, K\}$ (note that the rows of $\mJ$ are standard one-hot-encoded label vectors in $\R^K$).
We define the second-order data fluctuation vector as
\begin{equation*}
    \boldsymbol{\psi}\equiv\{\|\vx_i-\E[\vx_i]\|^2-\E[\|\vx_i-\E[\vx_i]\|^2]\}_{i=1}^{n}\in \R^{n},
\end{equation*}
and use
\begin{equation*}
    \mT = \{\tr \mC_a\mC_b/p\}_{a,b=1}^{K}\in \R^{K\times K}, \, \vt = \{\tr \mC_a^\circ/\sqrt{p}\}\in \R^K,
\end{equation*}
to denote the GMM second-order statistics. 
Also, let
\begin{equation}\label{eq:def_tau0}
    \tau_0\equiv\sqrt{\tr \mC^\circ/p},
\end{equation}
for $\mC^{\circ}\equiv\sum_{a=1}^{K}\frac{n_a}{n}\mC_a$ as in \Cref{def:gmm}, and $\tau_*$ be the fixed point to the following equation
\begin{equation}
    \tau_* = \sqrt{\sigma_a^2\E\left[\phi^2(\tau_*\xi)\right]+\sigma_b^2\tau_0^2}, \quad \xi\sim\gN(0,1),
    \label{eq:tau*}
\end{equation}
the existence and uniqueness of which is ensured under \Cref{cond:tau^*} as follows.



\begin{assum}\label{cond:tau^*}
For the DEQ model in \Cref{def:deq}, the variance parameter $\sigma_a$ satisfies $\sigma_a^2<2/(\E[(\phi^2(\tau\xi))''])$ for $\tau > 0$ and $\xi \sim \mathcal N(0,1)$.
\end{assum}
   
\begin{rem}[Existence and uniqueness of $\tau_*$]
It can be checked that for any given $\tau_0 > 0$ and variance parameter $\sigma_a$ satisfying \Cref{cond:tau^*}, 
the right-hand side of~\eqref{eq:tau*} constitutes a \emph{contractive mapping}. 
This ensures the existence of a unique fixed point $\tau_*$ in \eqref{eq:tau*}.
See \Cref{lem:tau*} in \Cref{sec:appendix_pre} for a detailed proof of this fact.
\end{rem}
With these notations, we are ready to present our first result on the high-dimensional characterization of CK matrices for implicit DEQs, the proof of which is given in \Cref{sec:proof_imCK}.
\begin{thm}[High-dimensional approximation of Implicit-CKs]
    For the DEQ model in \Cref{def:deq} with GMM input as in \Cref{def:gmm}, let Assumptions~\ref{assum:initial}~and~\ref{assum:activation} hold, and let the activation $\phi$ be centered such that $\E[\phi(\tau_*\xi)]=0$ for $\xi\sim\gN(0,1)$ and $\tau_*$ in \eqref{eq:tau*}.
    Further assume that the variance parameter $\sigma_a$ satisfies both Assumptions~\ref{cond:G*}~and~\ref{cond:tau^*}.
    Then, the Implicit-CK matrix $\mG^*$ defined in~\eqref{eq:imck} of \Cref{prop:Implicit-CK-NTK} can be well approximated, in a spectral norm sense with $\left\|\mG^*-\overline{\mG}\right\|=\order{n^{-1/2}}$, by a random matrix $\overline{\mG}$ \emph{explicitly} given by
    \begin{equation}
       \overline{\mG}\equiv \alpha_{*,1}\mX^\top\mX + \mV\mC_*\mV^\top+ (\gamma_*^2-\tau_0^2\alpha_{*,1})\mI_n,
       \label{eq:im-equiv}
    \end{equation}
    where $\mV = \left[\mJ/\sqrt{p},~\boldsymbol{\psi}\right]\in\R^{n\times (K+1)}$,
    \begin{equation*}
        \mC_* = \left[\begin{array}{cc}
        \alpha_{*,2}\vt\vt^\top +\alpha_{*,3}\mT& \alpha_{*,2}\vt  \\
        \alpha_{*,2}\vt^\top & \alpha_{*,2}
    \end{array}\right]\in\R^{(K+1)\times (K+1)},
\end{equation*}
and  non-negative scalars $\gamma_*, \alpha_{*,1}, \alpha_{*,2}, \alpha_{*,3}\geq 0$ are defined, for $\xi\sim\gN(0,1)$, as 
\begin{equation}\label{eq:alpha*}
    \begin{split}
    \gamma_* &= \sqrt{\E[\phi^2(\tau_*\xi)]},\,
    \alpha_{*,1}  = \frac{\sigma_b^2\E[\phi'(\tau_{*}\xi)]^2}{1-\sigma_a^2 \E[\phi'(\tau_{*}\xi)]^2}, \\
    \alpha_{*,2} &=\frac{\E[\phi''(\tau_{*}\xi)]^2}{4(1-\sigma_a^2 \E[\phi'(\tau_{*}\xi)]^2)}\alpha_{*,4}^2,\\
     \alpha_{*,3} &= \frac{\E[\phi''(\tau_{*}\xi)]^2}{2(1-\sigma_a^2 \E[\phi'(\tau_{*}\xi)]^2)}(\sigma_a^2\alpha_{*,1}+\sigma_b^2)^2
    \end{split}
\end{equation}
with $\alpha_{*,4} = (1-\frac{\sigma_a^2}{2}\E[(\phi^2(\tau_*\xi))''])^{-1}\sigma_b^2$.
\label{thm:imCK}
\end{thm}

Theorem~\ref{thm:imCK} reveals the surprising fact that, for high-dimensional GMM input in \Cref{def:gmm}, the Implicit-CK $\mG^*$, despite its mathematically involved form (as the fixed point of the recursion) in \eqref{eq:imck}, is close to a much simpler random matrix $\overline{\mG}$. 
This ``equivalent'' CK matrix $\overline{\mG}$, 
\begin{itemize}
    \item[(1)] depends, as expected, on the input GMM data ($\mX$), their class structure ($\mJ$) and higher-order statistics ($\vt$ and $\mT$), but in a rather \emph{explicit} fashion; and
    \item[(2)] is \emph{independent} of the distribution of the (randomly initialized) weight matrices $\mA$ and $\mB$; and 
    \item[(3)] depends on $\sigma_a^2$, $\sigma_b^2$, and the activation $\phi$ \emph{only} via four scalars $\alpha_{*,1}, \alpha_{*,2},\alpha_{*,3}, \gamma_*$ \emph{explicitly} given in \eqref{eq:alpha*}. 
\end{itemize}

A similar result can be derived for the Implicit-NTK matrices and is given as follows, proven in \Cref{sec:proof_imNTK}. 
\begin{thm}[High-dimensional approximation of Implicit-NTKs]
    Under the same settings and notations of \Cref{thm:imCK}, we have, that
    the Implicit-NTK matrix $\mK^*$ defined in~\eqref{eq:imntk} of \Cref{prop:Implicit-CK-NTK} can be well approximated, in a spectral norm sense with $\left\|\mK^*-\overline{\mK}\right\|=\order{n^{-1/2}}$, by a random matrix $\overline{\mK}$ \emph{explicitly} given by
    \begin{equation}
           \overline{\mK}\equiv \beta_{*,1}\mX^\top\mX + \mV\mD_*\mV^\top+ (\kappa_*^2-\tau_0^2\beta_{*,1})\mI_n,
    \end{equation}
    where $\mV \in\R^{n\times (K+1)}$ is as defined in Theorem~\ref{thm:imCK}, and
    \begin{align*}
        \mD_* &= \left[\begin{array}{cc}
        \beta_{*,2}\vt\vt^\top +\beta_{*,3}\mT& \beta_{*,2}\vt  \\
        \beta_{*,2}\vt^\top & \beta_{*,2}
                 \end{array}\right]\in\R^{(K+1)\times (K+1)},        
    \end{align*}
as well as non-negative scalars $\kappa_*, \beta_{*,1}, \beta_{*,2}, \beta_{*,3}\geq 0$,
\begin{equation*}
    \begin{split}
    \kappa_* & =\frac{\tau_*}{\sqrt{1-\sigma_a^2\E[\phi'(\tau_*\xi)^2]}},\
    \beta_{*,1}  = \frac{\alpha_{*,1}}{1-\sigma_a^2 \E[\phi'(\tau_*\xi)]^2}, \\
    \beta_{*,2} &=\frac{\alpha_{*,2}}{1-\sigma_a^2 \E[\phi'(\tau_*\xi)]^2},\\
     \beta_{*,3} &= \frac{\alpha_{*,3}+\beta_{*,1}(\sigma_a^2\E[\phi''(\tau_{*}\xi)]^2+\sigma_b^2)\alpha_{*,1}}{1-\sigma_a^2 \E[\phi'(\tau_*\xi)]^2},
    \end{split}
\end{equation*}
for $\xi\sim\gN(0,1)$, with $\alpha_{*,1}, \alpha_{*,2},\alpha_{*,3}$ as defined in \eqref{eq:alpha*}.
\label{thm:imNTK}
\end{thm}
\Cref{thm:imNTK} tells us that the NTK matrices of implicit DEQs take a similar form as their CKs, and (approximately for $n,p$ large) depend on $\sigma_a, \sigma_b$ and the activation via the key parameters $\beta_{*,1},\beta_{*,2},\beta_{*,3}$ and $\kappa_*$. 

\begin{rem}[On centered activation]
Given any activation function $\tilde\phi(\cdot)$ that satisfies Assumption~\ref{assum:activation}, a centered activation $\phi$ can be obtained by simplify subtracting a constant as $\phi(x) = \tilde\phi(x)-\E[\tilde\phi(\tau_* x)]$ with $\tau_* =\sqrt{\sigma_a^2\E[(\tilde\phi(\tau_*\xi) -\E[\tilde\phi(\tau_*\xi)])^2]+\sigma_b^2\tau_0^2}$.
\end{rem}

\subsection{High-dimensional Equivalence between DEQs and Shallow Explicit Networks}
\label{subsec:equivalence_implicit_explicit}

Implicit DEQs are known, per~\Cref{def:deq}, to be formally equivalent to \emph{infinitely} deep \emph{explicit} NN models~\citep{bai2020multiscale,xie2022optimization}.
In the sequel, we show how our theoretical analyses in Theorems~\ref{thm:imCK}~and~\ref{thm:imNTK} provide a general recipe to construct \emph{shallow explicit} NN models that are ``equivalent'' to a given \emph{implicit} DEQ model, in the sense that the CK and/or NTK matrices of the two networks are close in spectral norm.
We first review previous results on explicit NN models in \Cref{subsubsec:review_explicit}, and present, in \Cref{subsubsec:design}, guidelines to construct a shallow explicit NN equivalent to a given DEQ.

\subsubsection{A brief review of Explicit CKs and NTKs}
\label{subsubsec:review_explicit}

We consider the following $L$-layer \emph{explicit} NN model.
\begin{deft}[Fully-connected explicit NN model]\label{def:explicit_NN}
Let $\mX=[\vx_1,\cdots,\vx_n]\in \mathbb{R}^{p \times n}$ denote the input data, consider an $L$-layer fully-connected explicit NN model with output given by $\va^\top \vx_i^{(L)}$ for $\va \in \mathbb{R}^{m_L}$, $\vx_i^{(0)} = \vx_i$ and
\begin{equation}
    \vx_i^{(l)} = \frac{1}{\sqrt{m_l}}\sigma_l(\mW_l\vx_i^{(l-1)}), \quad \text{for $l=1,\cdots,L$},
\label{eq:exNN}    
\end{equation}
where $\mW_l\in\R^{m_{l}\times m_{l-1}}$ are weight matrices and $\sigma_l \colon \mathbb{R} \to \mathbb{R}$ are element-wise activation functions.
\end{deft}
As in Assumptions~\ref{assum:initial}~and~\ref{assum:activation} for implicit DEQs, we also assume that the weights $\mW_l$s in \Cref{def:explicit_NN} have \emph{i.i.d.}\@ entries of zero mean, unit variance, and finite fourth-order moment; and the activations $\sigma_l$ are four-times weakly differentiable with respect to standard Gaussian measure.

For fully-connected explicit NNs in \Cref{def:explicit_NN}, we recall the following result on their Explicit-CKs and NTKs.
\begin{prop}[Explicit-CKs and NTKs,~\citep{jacot2018neural,fan2020spectra}]\label{prop:Explicit-CK-NTK}
For a fully-connected $L$-layer NN model in \Cref{def:explicit_NN}, its Explicit-CK matrix $\bsigma^{(l)}$ at layer $l\in\{1,\cdots,L\}$ is defined as 
\begin{equation}
\bsigma_{ij}^{(l)}=\E_{\ru,\rv}[\sigma_l(\ru_l)\sigma_l(\rv_l)],  \quad \bsigma^{(0)}=\mX^\top\mX,
    \label{eq:exCK}
\end{equation}
with $(\ru_l,\rv_l)\sim\gN\biggl(0,\biggl[\begin{array}{cc}
    \bsigma_{ii}^{(l-1)}   &\bsigma_{ij}^{(l-1)}  \\
    \bsigma_{ji}^{(l-1)}     & \bsigma_{jj}^{(l-1)}
    \end{array}\biggr]\biggr)$.
And the Explicit-NTK matrix $\btheta^{(l)}$ at layer $l$ is defined as 
\begin{equation}
    \btheta^{(l)} = \bsigma^{(l)} + \btheta^{(l-1)}\odot\dot\bsigma^{(l)}, \quad \btheta^{(0)} = \mX^\top\mX,
    \label{eq:exNTK}
\end{equation}
with $\dot\bsigma_{ij}^{(l)} = \E_{\ru_l,\rv_l}[\sigma'_l(\ru_l)\sigma'_l(\rv_l)]$.
\end{prop}

The Explicit-CK and NTK matrices in \Cref{prop:Explicit-CK-NTK} for the fully-connected explicit NN model in \Cref{def:explicit_NN} have been recently studied in \citet{du2022lossless}.
\begin{thm}[High-dimensional approximation of Explicit-CKs,~{\citep[Theorem~1]{du2022lossless}}]~\label{thm:exCK}
    For fully-connected NN model in \Cref{def:explicit_NN} with GMM input in \Cref{def:gmm},
    let $\ttau_0=\tau_0$ as defined in \eqref{eq:def_tau0} and $\ttau_1,\cdots\ttau_L\geq0$ be a sequence of non-negative scalars satisfying $\ttau_l = \sqrt{\E[\sigma_l^2(\ttau_{l-1}\xi)]}$, for $\xi\sim\gN(0,1)$ and $l\in\{1,\cdots,L\}$. Further assume that the activation functions $\sigma_l(\cdot)$ are centered such that $\E[\sigma_l(\ttau_{l-1}\xi)]=0$. Then, for the Explicit-CK matrix $\bsigma^{(l)}$ defined in~\eqref{eq:exCK} of \Cref{prop:Explicit-CK-NTK}, it holds that $\|\bsigma^{(l)}-\overline{\bsigma}^{(l)}\|=\order{n^{-1/2}}$ for a random matrix $\overline{\bsigma}^{(l)}$ \emph{explicitly} given by
    \begin{equation}
        \overline{\bsigma}^{(l)} = \talpha_{l,1}\mX^\top\mX + \mV\tmC_l\mV^\top + (\ttau_l^2-\tau_0^2\talpha_1)\mI_n,
        \label{eq:ex-equiv}
    \end{equation}
   with $\mV \in\R^{n\times (K+1)}$ as defined in Theorem~\ref{thm:imCK}, 
    \begin{align*}
        \tmC_l &= \left[\begin{array}{cc}
        \talpha_{l,2}\vt\vt^\top +\talpha_{l,3}\mT& \talpha_{l,2}\vt  \\
        \talpha_{l,2}\vt^\top & \talpha_{l,2}
                 \end{array}\right]\in\R^{(K+1)\times (K+1)},        
    \end{align*}
     and non-negative scalars $\talpha_{l,1}, \talpha_{l,2}, \talpha_{l,3}$ defined recursively as $\talpha_{0,1}=\talpha_{0,4}=1$, $\talpha_{0,2}=\talpha_{0,3}=0$, and
\begin{equation*}
\begin{split}
    \talpha_{l,1} &= \E[\sigma'_l(\ttau_{l-1}\xi)]^2\talpha_{l-1,1},\\ 
    \talpha_{l,2} &=\E[\sigma'_l(\ttau_{l-1}\xi)]^2\talpha_{l-1,2} +\frac{1}{4} \E[\sigma''_l(\ttau_{l-1}\xi)]^2\talpha_{l-1,4}^2,\\
    \talpha_{l,3} &= \E[\sigma'_l(\ttau_{l-1}\xi)]^2\talpha_{l-1,3} + \frac{1}{2}\E[\sigma''_l(\ttau_{l-1}\xi)]^2\talpha_{l-1,1}^2, 
\end{split}
\end{equation*}
with  $\talpha_{l,4} = \E[(\sigma_l^2(\ttau_{l-1}\xi))'']\talpha_{l-1,4}$, for  $\xi\sim \mathcal N(0,1)$.
\end{thm}

In the following, we establish, by combining Theorems~\ref{thm:imCK}~and~\ref{thm:exCK}, \emph{explicit} connections between the CK matrices of implicit DEQs and fully-connected explicit NNs.
Exploiting this connection, we further provide a recipe to construct an explicit network ``equivalent'' to any given DEQ, with approximately the same CK.
Results for NTK can be similarly obtained by combining our \Cref{thm:imNTK} with \citet[Theorem~2]{du2022lossless} and is omitted here.

\subsubsection{Designing Equivalent Explicit NNs via CK matching}
\label{subsubsec:design}

Comparing \Cref{thm:imCK} to \Cref{thm:exCK}, we see that the high-dimensional approximation $\overline{\mG}$ of the Implicit-CK in~\eqref{eq:im-equiv} takes a consistent form with that ($\overline{\bsigma}^{(l)}$) of the Explicit-CK in~\eqref{eq:ex-equiv}, with coefficients $\alpha_*$s and $\tilde \alpha$s determined by the corresponding activation $\phi$ and $\sigma_l$, respectively. 
Inspired by this observation, our idea is to design activations of an $L$-layer explicit NN such that its Explicit-CK $\bsigma^{(L)}$ shares the same coefficients as the Implicit-CK $\mG^*$ of a given implicit DEQ of interest. 
Specifically, for a given implicit DEQ as in \Cref{def:deq}, 
\begin{enumerate}
    \item[(1)] we first compute the four key parameters $\alpha_{*,1}, \alpha_{*,2},\alpha_{*,3}$ and $\gamma_*$ of the implicit DEQ according to ~\eqref{eq:alpha*} of \Cref{thm:imCK};
    \item[(2)] we then select activations $\sigma_l$ with undetermined parameters for the $L$-layer explicit NN in~\Cref{def:explicit_NN}, and use \Cref{thm:exCK} to represent $\talpha_{L,1}, \talpha_{L,1}, \talpha_{L,1}, \ttau_L$ as functions of the activation parameters;
    \item[(3)] we determine the activations $\sigma_l$ of the explicit NN by solving the following set of equations,
    \begin{equation}
        \ttau_L=\gamma_*, \quad \talpha_{L,i}=\alpha_{*,i},~i\in\{1,2,3\}.
        \label{eq:coef_matching}
    \end{equation}
\end{enumerate}
This gives the desired fully-connected explicit NN model that shares the same CK as the given DEQ.

It remains to determine the depth of the equivalent explicit NN model. 
Note, by comparing \Cref{thm:exCK}~to~\Cref{thm:imCK}, that for a given implicit DEQ, it is \emph{not always possible} to determine a single-hidden-layer explicit NN having the same CK.
This is discussed in the following remark.

\begin{tcolorbox}
\begin{rem}[Implicit-~versus~Explicit-CK]\label{rem:implicit_CK_VS_explicit_CK}
It follows from \Cref{thm:exCK} that, for the single-hidden-layer explicit NN (with $L=1$ in \Cref{def:explicit_NN}), one \emph{must} have $\talpha_{1,2}=\frac{1}{2}\talpha_{1,3}$, \emph{regardless} of the choice of activation. 
On the contrast, $\alpha_{*,2}=\frac{1}{2}\alpha_{*,3}$ does \emph{not} necessarily hold for the Implicit-CK of \emph{all} DEQs. 
As such, for a given DEQ, 
\begin{itemize}
    \item if $\alpha_{*,2}=\frac{1}{2}\alpha_{*,3}$, then a single-hidden-layer explicit NN suffices to match the given DEQ;
    \item otherwise if $\alpha_{*,2}\neq \frac{1}{2}\alpha_{*,3}$, then an explicit NN with (at least) two hidden layers is required.
\end{itemize}  
\end{rem}
\end{tcolorbox}

As a consequence of \Cref{rem:implicit_CK_VS_explicit_CK}, we discuss, in the following, the two instances of commonly used implicit DEQ with ReLU and Tanh activations, and illustrate how to construct equivalent shallow explicit NNs in both cases.
The detailed expressions and proofs are given in \Cref{sec:proof_of_exm}.

\begin{exm}[DEQ with Tanh activation]\label{exm:Tanh}
    For a given implicit DEQ (denoted  {\sf Tanh-DEQ})  in \Cref{def:deq} with Tanh activation, \emph{i.e.}, $\phi(x) =\operatorname{Tanh}(x)$, a \textit{single-hidden-layer} equivalent explicit NN (denoted  {\sf H-Tanh-ENN}) as in~\Cref{def:explicit_NN}, with Hard-Tanh-type activation:
    \begin{equation}
        \sigma_{\operatorname{H-Tanh}}(x) \equiv ax \cdot 1_{-c\leq x\leq c}  + ac\cdot(1_{x\geq c} - 1_{ x\leq -c} ),
        \label{eq:htanh}
    \end{equation}
    with undetermined parameters $a>0, c\geq0$, can be constructed so that their CKs, denoted as $\mG_\text{Tanh}^*$ and $\bsigma_\text{H-Tanh}^{(1)}$, satisfy $\|\mG_\text{Tanh}^* - \bsigma_\text{H-Tanh}^{(1)}\| =\mathcal{O}(n^{-1/2}) $, by solving a system of nonlinear equations induced from~\eqref{eq:coef_matching}.
    \label{exm:tanh}
\end{exm}
\begin{exm}[DEQ with ReLU activation]\label{exm:ReLU}
    For a given implicit DEQ (denoted  {\sf ReLU-DEQ}) as in \Cref{def:deq} with centered ReLU activation, \emph{i.e.}, $\phi(x) =\operatorname{ReLU}(x)- \tau_*/\sqrt{2\pi}$,  a \textit{two-hidden-layer} equivalent explicit NN (denoted {\sf L-ReLU-ENN}) with Leaky-ReLU-type activation: 
    \begin{equation}
        \sigma_{\operatorname{L-ReLU}}^{(l)}(x) \equiv \max(a_lx,b_lx)-\frac{a_l-b_l}{\sqrt{2 \pi } }\ttau_l,\, l=1, 2,
        \label{eq:lrelu}
    \end{equation}
     with undetermined parameters $a_l\geq b_l\geq 0$, can be constructed so that their CKs, denoted as $\mG_\text{ReLU}^*$ and $\bsigma_\text{L-ReLU}^{(2)}$, satisfy $\| \mG_\text{ReLU}^* -  \bsigma_\text{L-ReLU}^{(2)} \| =\mathcal{O}(n^{-1/2}) $, by solving a system of polynomial equations induced from~\eqref{eq:coef_matching}. 
    \label{exm:RELU}
\end{exm}

\section{Experiments}
\begin{figure}
  \centering
       \begin{tikzpicture}[font=\large,spy using outlines, inner sep=1.2]
    \pgfplotsset{every major grid/.style={style=densely dashed}}
    \begin{axis}[
      height=.65\linewidth,
      width= 1\linewidth,
      xtick={0,300,600, 900, 1200},
      xticklabels = {{\small 0}, {\small 300}, {\small 600}, {\small 900} , {\small 1\,200}},
      ytick = {0.05, 0.10, 0.15, 0.20},
      yticklabels = {{\small 0.05},{\small 0.10},{\small 0.15},{\small 0.20}},
      grid=major,
      scaled ticks=true,
      xlabel={ $n$ },
      ylabel={{\small Relative error}},
      legend style = {at={(0.98,0.98)}, anchor=north east, font=\small}
      ]
      \addplot+[
  BLUE, mark=*,line width=0.75pt,
  error bars/.cd, 
    y fixed,
    y dir=both, 
    y explicit
]  table[x=X, y=Y, y error plus=ErrorMax, y error minus=ErrorMin] {
        X    Y                      ErrorMax            ErrorMin 
        40   0.1527771304030419   0.03136789981656002   0.02177147256867229
80   0.12065754599469518   0.018391676943001936   0.01536271586743325
160   0.08012528291622983   0.009433779763989666   0.009964632760276496
320   0.05286252816123446   0.0036472108198449538   0.004248821965089827
640   0.03298531594328939   0.0017274792373186826   0.0011436396490870915
1280   0.024043104050104288   0.0017305223401362785   0.001381426665842414
      };
      \addlegendentry{ {ReLU } };
      \addplot+[
  RED, mark=+,line width=0.75pt,
  error bars/.cd, 
    y fixed,
    y dir=both, 
    y explicit
]  table[x=X, y=Y, y error plus=ErrorMax, y error minus=ErrorMin] {
        X Y  ErrorMax ErrorMin
        40   0.1916980437900026   0.012682047632109172   0.01001326155605689
80   0.14533974987406412   0.03075710794598352   0.02097605500358457
160   0.131558997511012   0.013206313618998933   0.009281809978467864
320   0.09806974386202708   0.00469741175647527   0.004194139728127794
640   0.07409876832144856   0.0005428701399998737   0.0006442392076839581
1280   0.05451647563162624   0.0015092147937345937   0.0008573307644917302
      };
      \addlegendentry{ {Tanh } };
      \addplot+[
  PURPLE, mark=triangle*,line width=0.75pt,
  error bars/.cd, 
    y fixed,
    y dir=both, 
    y explicit
]  table[x=X, y=Y, y error plus=ErrorMax, y error minus=ErrorMin] {
        X Y  ErrorMax ErrorMin
        40   0.1864998036590179   0.011908059217664219   0.02188363537168095
80   0.11333000498825796   0.019707023224145775   0.01711579174979834
160   0.07049304179172138   0.004009483323273069   0.0036982419837571084
320   0.03914400158504383   0.002996276598674831   0.004529126722018231
640   0.029678433812390475   0.001815795020253145   0.0025290139201853724
1280   0.019801697331573   0.0007623603813714797   0.0012361559751191548
      };
      \addlegendentry{ {Swish } };
      \addplot+[
  GREEN, mark=diamond*,line width=0.75pt,
  error bars/.cd, 
    y fixed,
    y dir=both, 
    y explicit
]  table[x=X, y=Y, y error plus=ErrorMax, y error minus=ErrorMin] {
        X Y  ErrorMax ErrorMin
        40   0.1562899857077683   0.0023306462040554587   0.0015986618287569254
80   0.12510744626635698   0.008125965403635999   0.005945602343283202
160   0.0736241370435096   0.00798236196839959   0.0060100869280101665
320   0.0511286000813501   0.005018745187890347   0.0028009446361789148
640   0.0330772041056896   0.001669354685004211   0.001367162666707139
1280   0.02232802810475876   0.0005512452101586185   0.00044990729225136486
      };
      \addlegendentry{ {L-ReLU } };
    \end{axis}
  \end{tikzpicture}
      
\caption{Evolution of relative  spectral norm  error $\|\mG^*-\overline{\mG}\|/\|\mG^*\|$  \emph{w.r.t.}\@ sample size $n$, for DEQs in \Cref{def:deq} with different activations and $\sigma_a^2=0.2$, on two-class GMM, $p/n = 0.8$, $\vmu_a=[\mathbf{0}_{8(a-1)};8;\mathbf{0}_{p-8a+7}]$, and $\mC_a=(1+8(a-1)/\sqrt{p})\mI_p, a\in \{1,2\}$.
 }
\label{fig:thm1}
\end{figure}
In this section, we provide numerical experiments to validate our theoretical results.
We consider both Gaussian mixture data and samples drawn from commonly used real-world datasets such as MNIST~\citep{lecun1998gradient}, Fashion-MNIST~\citep{xiao2017fashion}, and CIFAR-10~\citep{krizhevsky2009learning}. 
The experiments are conducted with a repetition of five trials, and we report both the average performance and accompanying error bars.
Due to space limitation, we refer the readers to \Cref{sec:additional_expri} for additional experiments. 
The code to reproduce the results in this section is available at \url{https://github.com/StephenLi24/INN_eqvi_ENN}.
%
%
%
\begin{figure}
  \centering
  \resizebox{\columnwidth}{!}{
  \centering
  \begin{tabular}{@{}c@{}}
  \hspace{3pt}
  \begin{tabular}{@{}cc@{}}
    \begin{tikzpicture}[font=\small,spy using outlines, inner sep=0]
    \renewcommand{\axisdefaulttryminticks}{4} 
    \pgfplotsset{every major grid/.append style={densely dashed}}  
    \tikzstyle{every axis y label}+=[yshift=-10pt] 
    \tikzstyle{every axis x label}+=[yshift=5pt]
    \pgfplotsset{every axis legend/.style={cells={anchor=east},fill=none,at={(1,1.12)}, legend columns = -1, anchor=north east, font=\small}}
    \begin{axis}[
      legend style={
        legend image code/.code={ 
            \draw[##1,line width=1.5pt] plot coordinates {(0cm,0cm) (0.4cm,0cm)}; 
        }
    },
      height=.45\columnwidth,
      width=.5\columnwidth,
      xmin=-2,xmax=2,
      xtick={0},
      xticklabels = { {\small 0} },
      ytick={-1.5, -1,0,1, 1.5},
      yticklabels = {{\small$-ac$},{\small-1},{\small 0}, {\small 1}, {\small $ac$}},
      grid=major,
      scaled ticks=true,
      xlabel={ \phantom{x} },
      bar width=3pt,
      ]
              \addplot[densely dashed, samples=3000,domain=-2.5:2,RED,line width=1.5pt] 
              {tanh(x)};;
              \addlegendentry{Tanh}

              \addplot[samples=3000,domain=-2.5:2,RED,line width=1pt] 
              {max(-1.5, min(1.5, x))};;
              \addlegendentry{H-Tanh}
      \end{axis}
\end{tikzpicture}
             &
     \begin{tikzpicture}[font=\large,spy using outlines, inner sep=0]
    \renewcommand{\axisdefaulttryminticks}{4} 
    \pgfplotsset{every major grid/.append style={densely dashed}}       
    \tikzstyle{every axis y label}+=[yshift=-10pt] 
    \tikzstyle{every axis x label}+=[yshift=5pt]
        \pgfplotsset{every axis legend/.style={cells={anchor=east},fill=none,at={(1,1.21)}, legend columns = -1, anchor=north east, font=\small}}
    \begin{axis}[
      height=.45\columnwidth,
      width=.6\columnwidth,
      ytick={0,0.01,0.02},
      yticklabels={{\small 0}, {\small 0.01}, {\small 0.02}},
      xtick={0,400, 800, 1200},
      xticklabels={{\small 0}, {\small 400}, {\small 800},{\small1\,200}},
      bar width=3pt,
      grid=major,
      scaled ticks=false,
      xlabel={{\small $n$} },
      ylabel={{\small Relative error}},
      ylabel style={yshift=+0.1cm}
      ]
    \addplot+[
    RED, mark=triangle, line width=1pt,
    error bars/.cd, 
      y fixed,
      y dir=both, 
      y explicit
  ]  table[x=X, y=Y, y error plus=ErrorMax, y error minus=ErrorMin, row sep=crcr] {
          X Y  ErrorMax ErrorMin\\
          40   0.019795981680829525   0.004162337292034076   0.002569096373338974\\
  80   0.010569388428630333   0.001771139368292209   0.001273339186273828\\
  160   0.009099677265216261   0.0025779444384222538   0.0015452792283098477\\
  320   0.006653496668660815   0.0010608762411369868   0.001579221903142784\\
  640   0.004235081499690153   0.0002866678191454598   0.00023643345374468015\\
  1280   0.0028800766066381946   6.467964823835687e-05   6.921189176735051e-05\\
   };
        \addlegendentry{ $\mG_\text{Tanh}^*$ v.s. $ \bsigma_\text{H-Tanh}^{(1)}$};
\end{axis}
\end{tikzpicture}
  \end{tabular}
  
\\

  \begin{tabular}{@{}cc@{}}
  
    \begin{tikzpicture}[font=\small,spy using outlines, inner sep=0]
    \renewcommand{\axisdefaulttryminticks}{4} 
    \pgfplotsset{every major grid/.append style={densely dashed}}  
    \tikzstyle{every axis y label}+=[yshift=-10pt] 
    \tikzstyle{every axis x label}+=[yshift=5pt]
    \pgfplotsset{every axis legend/.style={cells={anchor=east},fill=none,at={(1,1.12)}, legend columns = -1, anchor=north east, font=\small}}
    \begin{axis}[
      legend style={
        legend image code/.code={ 
            \draw[##1,line width=1.5pt] plot coordinates {(0cm,0cm) (0.3cm,0cm)}; 
        }
    },
      height=.45\columnwidth,
      width=.5\columnwidth,
      xmin=-2,xmax=2,
      xtick={ 0},
      xticklabels = { {\small 0} },
      ytick={ -1/sqrt(2*pi), 0},
      yticklabels = { {\small $-\frac{\tau_*}{ \sqrt{2\pi}}$}, {\small $0$} },
      grid=major,
      scaled ticks=true,
      xlabel={ \phantom{x} },
      bar width=3pt
      ]
              \addplot[densely dashed, samples=3000,domain=-2.5:2,BLUE,line width=1.5pt] 
              { max(0,x)- 1/sqrt(2*pi) };
              \addlegendentry{ReLU}
              \addplot[samples=3000,domain=-2.5:2,BLUE,line width=1pt] 
              { max(1.5*x,.2*x) - 1.3/sqrt(2*pi) };
              \addlegendentry{L-ReLU}
      \end{axis}
\end{tikzpicture}
        &
     \begin{tikzpicture}[font=\large,spy using outlines, inner sep=0]
    \renewcommand{\axisdefaulttryminticks}{4} 
    \pgfplotsset{every major grid/.append style={densely dashed}}       
    \tikzstyle{every axis y label}+=[yshift=-10pt, xshift=-10pt] 
    \tikzstyle{every axis x label}+=[yshift=5pt]
    \pgfplotsset{every axis legend/.style={cells={anchor=east},fill=none,at={(1,1.21)}, legend columns = -1, anchor=north east, font=\small}}
    \begin{axis}[
      height=.45\columnwidth,
      width=.6\columnwidth,
      ytick={0,0.05,0.1},
      yticklabels={{\small0},{\small0.05},{\small0.10}},
      xtick={0,400, 800, 1200},
      xticklabels={{\small 0}, {\small 400}, {\small 800},{\small1\,200}},
      bar width=3pt,
      grid=major,
      scaled ticks=true,
      xlabel={{\small $n$} },
      ylabel={{\small Relative error} },
      ylabel style={yshift=+0.1cm}
      ]
    \addplot+[
    BLUE, mark=o,line width=1pt,
    error bars/.cd, 
      y fixed,
      y dir=both, 
      y explicit
  ]  table[x=X, y=Y, y error plus=ErrorMax, y error minus=ErrorMin, row sep=crcr] {
          X Y  ErrorMax ErrorMin\\
          40   0.07961757634023776   0.02582411973382348   0.018458118838325627\\
  80   0.036605936360673406   0.009196506146308812   0.0061597327817537045\\
  160   0.02181025395516543   0.0007018890501470053   0.0006980013136775667\\
  320   0.013959464653985393   0.0013376190442611017   0.0011853010846024862\\
  640   0.010087775626202342   0.0007397558030735694   0.0013044326612726385\\
  1280   0.0062326440774518556   0.00015195060486591656   0.0002139450388602632\\
        };
        \addlegendentry{ $\mG_\text{ReLU}^*$ v.s. $\bsigma_\text{L-ReLU}^{(2)}$};
\end{axis}
\end{tikzpicture}
  \end{tabular}
  \end{tabular}
  }
\caption{\textbf{Left:} Visualization of activations  of DEQs (\textbf{dashed}) and those of  equivalent explicit NNs (\textbf{solid}). \textbf{Right:} Evolution of relative spectral norm errors  $\|\mG_{\text{Tanh}}^*-\bsigma_{\text{H-Tanh}}^{(1)}\|/ \|\mG_{\text{Tanh}}^*\|$ and $\|\mG_{\text{ReLU}}^*-\bsigma_{\text{L-ReLU}}^{(2)}\|/ \|\mG_{\text{ReLU}}^*\|$  \emph{w.r.t.}  sample size $n$ on GMM as in \Cref{fig:thm1} for Example~\ref{exm:Tanh} (\textbf{\RED red}) and  Example~\ref{exm:ReLU}  (\textbf{\BLUE blue}), respectively.} 
\label{fig:equiv}
\end{figure}
\input{fig/NN_matching_sgd}
\paragraph{High-dimensional approximations of Implicit-CKs and NTKs.} 
\Cref{fig:thm1} compares the difference between Implicit-CKs $\mG^*$ and their high-dimensional approximations $\overline{\mG}$ given in \Cref{thm:imCK}, on binary Gaussian mixture data, for DEQs as \Cref{def:deq} with four commonly-used activations: ReLU, Tanh, Swish, and Leaky-ReLU (L-ReLU). 
The computation of $\overline{\mG}$ follows from its definition in \Cref{thm:imCK}.  
For the Implicit-CK $\mG^*$, we take an estimation approach similar to that in~\citet{gao2023wide}: each element $\mG_{ij}^*$  is estimated as $(\vz_i^{*})^\top\vz_j^{*}$ using a high-dimensional DEQ defined in~\eqref{eq:deq} with $m=2^{12}$ and $\vz_i^{*}$ estimated through a large number $l$ of fixed-point iterations defined in~\eqref{eq:deql}.
See~\citet{gao2023wide} for a convergence analysis of this estimation (to $\mG^*$) \emph{w.r.t.} the width $m$ and the number $l$ of fixed-point iterations. 
We refer the interested readers to~\citet{cho2009kernel,tsuchida2018invariance,novak2019neural} for fast and efficient estimation/computation of CKs and NTKs.

We observe from \Cref{fig:thm1} that, for different activations, as $n,p$ increase, the relative errors consistently and significantly decrease, as in line with our \Cref{thm:imCK}.
The experimental observations regarding NTKs and \Cref{thm:imNTK} are similar and are placed in~\Cref{sm: real-imp}.
Possibly surprisingly, the high-dimensional approximations of Implicit-CKs and Implicit-NTKs, despite derived here for GMM in Theorems~\ref{thm:imCK}~and~\ref{thm:imNTK}, exhibits unexpected similar behavior on realistic MNIST data, see~\Cref{sm:realCKNTK} for detailed results.

\paragraph{Equivalent Explicit-CKs and NTKs.}
In \Cref{fig:equiv}, we testify the results in Examples~\ref{exm:Tanh}~and~\ref{exm:ReLU} by constructing shallow \emph{explicit} networks with Hard Tanh-type ({\sf H-Tanh-ENN}) and Leaky ReLU-type ({\sf L-ReLU-ENN}) activation equivalent to implicit DEQs with Tanh ({\sf Tanh-DEQ}) and ReLU ({\sf ReLU-DEQ}) activation, respectively.
We see that, while the two types of NN models are different in that
(\romannumeral1) {\sf DEQ}s are implicitly defined while {\sf ENN}s are explicitly defined, and
(\romannumeral2) {\sf ENN}s use different activations from {\sf DEQ}s,
their CK matrices are close in spectral norm, as long as the activation of {\sf ENN}s are carefully chosen according to our Examples~\ref{exm:Tanh}~and~\ref{exm:ReLU}. 
This observation is again consistent on synthetic GMM, \emph{and} possibly surprisingly, realistic MNIST data. 
We conjecture that this is due to a high-dimensional \emph{universal} phenomenon and that our results hold more generally beyond GMM for, say, data drawn from the family of concentrated random vectors~\citep{ledoux2005concentration}.
We refer the interested readers to \citet[Chapter~8]{couillet2022RMT4ML} for more discussions on this point.
%
%

\paragraph{Test performance of explicit NNs on realistic data.}
To explore the extent of the proposed high-dimensional equivalence between implicit DEQs and shallow explicit NN models across various realistic datasets, we conduct a comprehensive comparison of the classification accuracies using both implicit and explicit models. 
The results of this comparison, depicted in \Cref{fig:matching_sgd}, provide insights into the performance of DEQs against carefully (or not) designed explicit NNs.
Following Examples~\ref{exm:Tanh}~and~\ref{exm:ReLU},  we construct a single-hidden-layer {\sf H-Tanh-ENN} and a two-hidden-layer {\sf L-ReLU-ENN} to match {\sf Tanh-DEQ} and {\sf ReLU-DEQ}, respectively. 
The undetermined parameters $a,c$ and $a_l,b_l$ of the activations {\sf H-Tanh-ENN} and {\sf L-ReLU-ENN} are determined by solving the system of equations induced from~\eqref{eq:coef_matching}. 
For comparison, we also construct a single-hidden-layer explicit NN with Tanh activation (denoted {\sf Tanh-ENN}) and a two-hidden-layer explicit NN with ReLU activation (denoted {\sf ReLU-ENN}). 
Models are trained using SGD optimizer with learning rates of $10^{-1}$ for MNIST and Fashion-MNIST, and $10^{-2}$ for CIFAR-10.
The batch size is set to $128$ with a maximum training epoch of $100$.
To ensure a fair comparison, the hidden layer of explicit NNs share the \emph{same} width, $m \in 2^{5-12}$, as the implicit DEQs.
As $m$ increases, the performance of {\sf L-ReLU-ENN} closely matches that of  {\sf ReLU-DEQ}, while a noticeable performance gap exists between {\sf ReLU-ENN} and  {\sf ReLU-DEQ}. 
A similar result is observed in the case of {\sf H-Tanh-ENN} and  {\sf Tanh-DEQ}. 
These trends are in line with the theoretical guaranteed offered by our analysis, that focuses on CKs and NTKs and formally holds in the $m \to \infty$ limit. 
Experiments are also conducted using the Adam optimizer, where similar trends can be observed. 
Please refer to \Cref{sm:adam_results} for these results.
Moreover, we  observe the remarkable advantage of {\sf ENN} over {\sf DEQ}s on the time costs of inference and training, see~\Cref{sm:timecost} for detailed results.
This observation substantiates our theory and underscores the practical advantages of our approach by,\emph{ e.g.,} enabling the design of memory-efficient explicit NNs that achieve the performance of implicit DEQs without the computational overhead associated with fixed-point iterations.

\section{Conclusion}

In this paper, we investigate the connections and differences between implicit DEQs and explicit NNs. We employ RMT to analyze the eigenspectra of the NTKs and CKs of implicit DEQs. For high-dimensional Gaussian mixture data, we establish high-dimensional approximations for the NTK and CK of implicit DEQs. Notably, we reveal that the eigenspectra of the NTK and CK of implicit DEQs are determined solely by the variance parameter and the activation function.
Based on this observation, we establish the equivalence between implicit DEQs and explicit NNs in high dimensions. We propose a method for designing activation functions for explicit neural networks to match the spectral behavior of the CK (or NTK) of implicit DEQs. 
Results on GMM data and real-world data demonstrate that shallow explicit NNs using our theoretically designed activation functions achieve comparable performance to implicit DEQs, with significantly reduced computational overhead.

\section*{Impact Statement}

This paper presents work whose goal is to advance the field of Machine Learning. 
There are many potential societal consequences of our work, none of which we feel must be specifically highlighted here.

\section*{Acknowledgements}
Z.~Liao would like to acknowledge the National Natural Science Foundation of China (via NSFC-62206101) and the Fundamental Research Funds for the Central Universities of China (2021XXJS110), for providing partial support.
Z.~Liao and Z.~Ling would like to acknowledge the Guangdong Key Lab of Mathematical Foundations for Artificial Intelligence Open Fund (OFA00003).
R.~C.~Qiu and Z.~Liao  would like to acknowledge the National Natural Science Foundation of China (via NSFC-12141107), the Interdisciplinary Research Program of HUST (2023JCYJ012), the Key Research and Development Program of Guangxi (GuiKe-AB21196034). 
F.~Zhou  was supported by the National Natural Science Foundation of China Project (via NSFC-62106121) and  the MOE Project of Key Research Institute of Humanities and Social Sciences (22JJD110001).

\bibliography{example_paper}

\begin{thebibliography}{51}
\providecommand{\natexlab}[1]{#1}
\providecommand{\url}[1]{\texttt{#1}}
\expandafter\ifx\csname urlstyle\endcsname\relax
  \providecommand{\doi}[1]{doi: #1}\else
  \providecommand{\doi}{doi: \begingroup \urlstyle{rm}\Url}\fi

\bibitem[Agarwala \& Schoenholz(2022)Agarwala and Schoenholz]{agarwala2022deep}
Agarwala, A. and Schoenholz, S.~S.
\newblock Deep equilibrium networks are sensitive to initialization statistics.
\newblock In \emph{International Conference on Machine Learning}, pp.\
  136--160. PMLR, 2022.

\bibitem[Alemohammad et~al.(2020)Alemohammad, Wang, Balestriero, and
  Baraniuk]{alemohammad2020recurrent}
Alemohammad, S., Wang, Z., Balestriero, R., and Baraniuk, R.
\newblock The recurrent neural tangent kernel.
\newblock \emph{arXiv preprint arXiv:2006.10246}, 2020.

\bibitem[Ali et~al.(2022)Ali, Liao, and Couillet]{ali2021random}
Ali, H.~T., Liao, Z., and Couillet, R.
\newblock Random matrices in service of ml footprint: ternary random features
  with no performance loss.
\newblock \emph{ICLR}, 2022.

\bibitem[Arora et~al.(2019)Arora, Du, Hu, Li, Salakhutdinov, and
  Wang]{arora2019exact}
Arora, S., Du, S.~S., Hu, W., Li, Z., Salakhutdinov, R.~R., and Wang, R.
\newblock On exact computation with an infinitely wide neural net.
\newblock In \emph{Advances in Neural Information Processing Systems}, pp.\
  8141--8150, 2019.

\bibitem[Bai et~al.(2019)Bai, Kolter, and Koltun]{NEURIPS2019_01386bd6}
Bai, S., Kolter, J.~Z., and Koltun, V.
\newblock Deep equilibrium models.
\newblock In \emph{Advances in Neural Information Processing Systems},
  volume~32. Curran Associates, Inc., 2019.

\bibitem[Bai et~al.(2020)Bai, Koltun, and Kolter]{bai2020multiscale}
Bai, S., Koltun, V., and Kolter, J.~Z.
\newblock Multiscale deep equilibrium models.
\newblock \emph{Advances in Neural Information Processing Systems}, 2020.

\bibitem[Bai et~al.(2021)Bai, Koltun, and Kolter]{bai2021neural}
Bai, S., Koltun, V., and Kolter, J.~Z.
\newblock Neural deep equilibrium solvers.
\newblock In \emph{International Conference on Learning Representations}, 2021.

\bibitem[Bartlett et~al.(2021)Bartlett, Montanari, and
  Rakhlin]{bartlett2021Deep}
Bartlett, P.~L., Montanari, A., and Rakhlin, A.
\newblock Deep learning: A statistical viewpoint.
\newblock \emph{Acta Numerica}, 30:\penalty0 87--201, May 2021.
\newblock ISSN 0962-4929, 1474-0508.
\newblock \doi{10.1017/S0962492921000027}.

\bibitem[Benigni \& P{\'e}ch{\'e}(2019)Benigni and
  P{\'e}ch{\'e}]{benigni1904eigenvalue}
Benigni, L. and P{\'e}ch{\'e}, S.
\newblock Eigenvalue distribution of nonlinear models of random matrices.
\newblock \emph{arXiv preprint arXiv:1904.03090}, 2019.

\bibitem[Blum et~al.(2020)Blum, Hopcroft, and Kannan]{blum2020foundations}
Blum, A., Hopcroft, J., and Kannan, R.
\newblock \emph{Foundations of Data Science}.
\newblock {Cambridge University Press}, 2020.
\newblock ISBN 978-1-108-48506-7.
\newblock \doi{10.1017/9781108755528}.

\bibitem[Cho \& Saul(2009)Cho and Saul]{cho2009kernel}
Cho, Y. and Saul, L.
\newblock Kernel methods for deep learning.
\newblock \emph{Advances in neural information processing systems}, 22, 2009.

\bibitem[Couillet \& {Benaych-Georges}(2016)Couillet and
  {Benaych-Georges}]{couillet2016kernel}
Couillet, R. and {Benaych-Georges}, F.
\newblock Kernel spectral clustering of large dimensional data.
\newblock \emph{Electronic Journal of Statistics}, 10\penalty0 (1):\penalty0
  1393--1454, 2016.
\newblock ISSN 1935-7524.
\newblock \doi{10.1214/16-ejs1144}.

\bibitem[Couillet \& Liao(2022)Couillet and Liao]{couillet2022RMT4ML}
Couillet, R. and Liao, Z.
\newblock \emph{Random Matrix Methods for Machine Learning}.
\newblock Cambridge University Press, 2022.
\newblock ISBN 9781009186742.

\bibitem[Dobriban \& Wager(2018)Dobriban and Wager]{dobriban2018high}
Dobriban, E. and Wager, S.
\newblock High-dimensional asymptotics of prediction: {{Ridge}} regression and
  classification.
\newblock \emph{The Annals of Statistics}, 46\penalty0 (1):\penalty0 247--279,
  2018.
\newblock ISSN 0090-5364.
\newblock \doi{10.1214/17-aos1549}.

\bibitem[Du et~al.(2019)Du, Hou, Salakhutdinov, Poczos, Wang, and
  Xu]{du2019graph}
Du, S.~S., Hou, K., Salakhutdinov, R.~R., Poczos, B., Wang, R., and Xu, K.
\newblock Graph neural tangent kernel: Fusing graph neural networks with graph
  kernels.
\newblock \emph{Advances in neural information processing systems}, 32, 2019.

\bibitem[El~Ghaoui et~al.(2021)El~Ghaoui, Gu, Travacca, Askari, and
  Tsai]{el2021implicit}
El~Ghaoui, L., Gu, F., Travacca, B., Askari, A., and Tsai, A.
\newblock Implicit deep learning.
\newblock \emph{SIAM Journal on Mathematics of Data Science}, 3\penalty0
  (3):\penalty0 930--958, 2021.

\bibitem[Elkhalil et~al.(2020)Elkhalil, Kammoun, Couillet, {Al-Naffouri}, and
  Alouini]{elkhalil2020large}
Elkhalil, K., Kammoun, A., Couillet, R., {Al-Naffouri}, T.~Y., and Alouini,
  M.-S.
\newblock A large dimensional study of regularized discriminant analysis.
\newblock \emph{IEEE Transactions on Signal Processing}, 68:\penalty0
  2464--2479, 2020.
\newblock ISSN 1053-587X.
\newblock \doi{10.1109/tsp.2020.2984160}.

\bibitem[Fan \& Wang(2020)Fan and Wang]{fan2020spectra}
Fan, Z. and Wang, Z.
\newblock Spectra of the conjugate kernel and neural tangent kernel for
  linear-width neural networks.
\newblock \emph{Advances in neural information processing systems},
  33:\penalty0 7710--7721, 2020.

\bibitem[Feng \& Kolter(2020)Feng and Kolter]{feng2020neural}
Feng, Z. and Kolter, J.~Z.
\newblock On the neural tangent kernel of equilibrium models.
\newblock \emph{arxiv}, 2020.

\bibitem[Fung et~al.(2022)Fung, Heaton, Li, McKenzie, Osher, and
  Yin]{fung2022jfb}
Fung, S.~W., Heaton, H., Li, Q., McKenzie, D., Osher, S., and Yin, W.
\newblock Jfb: Jacobian-free backpropagation for implicit networks.
\newblock In \emph{Proceedings of the AAAI Conference on Artificial
  Intelligence}, volume~36, pp.\  6648--6656, 2022.

\bibitem[Gao \& Gao(2022)Gao and Gao]{gao2022}
Gao, T. and Gao, H.
\newblock Gradient descent optimizes infinite-depth relu implicit networks with
  linear widths.
\newblock \emph{arxiv}, 2022.

\bibitem[Gao et~al.(2022)Gao, Liu, Liu, Rajan, and Gao]{2021A}
Gao, T., Liu, H., Liu, J., Rajan, H., and Gao, H.
\newblock A global convergence theory for deep relu implicit networks via
  over-parameterization.
\newblock \emph{ICLR}, 2022.

\bibitem[Gao et~al.(2023)Gao, Huo, Liu, and Gao]{gao2023wide}
Gao, T., Huo, X., Liu, H., and Gao, H.
\newblock Wide neural networks as gaussian processes: Lessons from deep
  equilibrium models.
\newblock \emph{NeurIPS}, 2023.

\bibitem[Gilton et~al.(2021)Gilton, Ongie, and Willett]{gilton2021deep}
Gilton, D., Ongie, G., and Willett, R.
\newblock Deep equilibrium architectures for inverse problems in imaging.
\newblock \emph{arXiv preprint arXiv:2102.07944}, 2021.

\bibitem[Goodfellow et~al.(2016)Goodfellow, Bengio, and
  Courville]{goodfellow2016deep}
Goodfellow, I., Bengio, Y., and Courville, A.
\newblock Deep learning mit press (2016).
\newblock In \emph{Conference on information and communication systems
  (ICICS)}, pp.\  151--156, 2016.

\bibitem[Gu et~al.(2022)Gu, Du, Yuan, Xie, Pu, Qiu, and Liao]{du2022lossless}
Gu, L., Du, Y., Yuan, Z., Xie, D., Pu, S., Qiu, R., and Liao, Z.
\newblock " lossless" compression of deep neural networks: A high-dimensional
  neural tangent kernel approach.
\newblock \emph{Advances in Neural Information Processing Systems},
  35:\penalty0 3774--3787, 2022.

\bibitem[Huang et~al.(2021)Huang, Bai, and Kolter]{huang2021textrm}
Huang, Z., Bai, S., and Kolter, J.~Z.
\newblock $(\textrm{Implicit})^2$: Implicit layers for implicit
  representations.
\newblock In \emph{Advances in Neural Information Processing Systems},
  volume~34, pp.\  9639--9650, 2021.

\bibitem[Jacot et~al.(2018)Jacot, Gabriel, and Hongler]{jacot2018neural}
Jacot, A., Gabriel, F., and Hongler, C.
\newblock Neural tangent kernel: convergence and generalization in neural
  networks.
\newblock In \emph{Advances in Neural Information Processing Systems}, pp.\
  8580--8589, 2018.

\bibitem[Krizhevsky(2009)]{krizhevsky2009learning}
Krizhevsky, A.
\newblock Learning multiple layers of features from tiny images.
\newblock pp.\  32--33, 2009.
\newblock URL
  \url{https://www.cs.toronto.edu/~kriz/learning-features-2009-TR.pdf}.

\bibitem[LeCun et~al.(1998)LeCun, Bottou, Bengio, and
  Haffner]{lecun1998gradient}
LeCun, Y., Bottou, L., Bengio, Y., and Haffner, P.
\newblock Gradient-based learning applied to document recognition.
\newblock \emph{Proceedings of the IEEE}, 86\penalty0 (11):\penalty0
  2278--2324, 1998.

\bibitem[Ledoux(2005)]{ledoux2005concentration}
Ledoux, M.
\newblock \emph{{The Concentration of Measure Phenomenon}}.
\newblock Mathematical Surveys and Monographs. American Mathematical Soc.,
  2005.
\newblock ISBN 9780821837924.
\newblock \doi{10.1090/surv/089}.

\bibitem[Liao \& Couillet(2018{\natexlab{a}})Liao and
  Couillet]{liao2018dynamics}
Liao, Z. and Couillet, R.
\newblock The dynamics of learning: A random matrix approach.
\newblock In \emph{International Conference on Machine Learning}, pp.\
  3072--3081. PMLR, 2018{\natexlab{a}}.

\bibitem[Liao \& Couillet(2018{\natexlab{b}})Liao and
  Couillet]{liao2018spectrum}
Liao, Z. and Couillet, R.
\newblock On the spectrum of random features maps of high dimensional data.
\newblock In \emph{International Conference on Machine Learning}, pp.\
  3063--3071. PMLR, 2018{\natexlab{b}}.

\bibitem[Liao \& Couillet(2019)Liao and Couillet]{liao2019large}
Liao, Z. and Couillet, R.
\newblock A large dimensional analysis of least squares support vector
  machines.
\newblock \emph{IEEE Transactions on Signal Processing}, 67\penalty0
  (4):\penalty0 1065--1074, 2019.
\newblock ISSN 1053-587X.
\newblock \doi{10.1109/tsp.2018.2889954}.

\bibitem[Ling et~al.(2023)Ling, Xie, Wang, Zhang, and Lin]{ling2023global}
Ling, Z., Xie, X., Wang, Q., Zhang, Z., and Lin, Z.
\newblock Global convergence of over-parameterized deep equilibrium models.
\newblock In \emph{International Conference on Artificial Intelligence and
  Statistics}, pp.\  767--787. PMLR, 2023.

\bibitem[Louart \& Couillet(2018)Louart and Couillet]{louart2019concentration}
Louart, C. and Couillet, R.
\newblock {Concentration of Measure and Large Random Matrices with an
  application to Sample Covariance Matrices}.
\newblock \emph{arXiv}, 2018.
\newblock URL \url{https://arxiv.org/pdf/1805.08295}.

\bibitem[Louart et~al.(2018)Louart, Liao, and Couillet]{louart2018random}
Louart, C., Liao, Z., and Couillet, R.
\newblock A random matrix approach to neural networks.
\newblock \emph{Annals of Applied Probability}, 28\penalty0 (2):\penalty0
  1190--1248, 2018.
\newblock ISSN 1050-5164.
\newblock \doi{10.1214/17-aap1328}.

\bibitem[Mei \& Montanari(2022)Mei and Montanari]{mei2022generalization}
Mei, S. and Montanari, A.
\newblock The generalization error of random features regression: Precise
  asymptotics and the double descent curve.
\newblock \emph{Communications on Pure and Applied Mathematics}, 75\penalty0
  (4):\penalty0 667--766, 2022.

\bibitem[Micaelli et~al.(2023)Micaelli, Vahdat, Yin, Kautz, and
  Molchanov]{micaelli2023recurrence}
Micaelli, P., Vahdat, A., Yin, H., Kautz, J., and Molchanov, P.
\newblock Recurrence without recurrence: Stable video landmark detection with
  deep equilibrium models.
\newblock In \emph{Proceedings of the IEEE/CVF Conference on Computer Vision
  and Pattern Recognition}, pp.\  22814--22825, 2023.

\bibitem[Novak et~al.(2019)Novak, Xiao, Hron, Lee, Alemi, Sohl-Dickstein, and
  Schoenholz]{novak2019neural}
Novak, R., Xiao, L., Hron, J., Lee, J., Alemi, A.~A., Sohl-Dickstein, J., and
  Schoenholz, S.~S.
\newblock Neural tangents: Fast and easy infinite neural networks in python.
\newblock \emph{arXiv preprint arXiv:1912.02803}, 2019.

\bibitem[Pastur(2022)]{pastur2022eigenvalue}
Pastur, L.
\newblock Eigenvalue distribution of large random matrices arising in deep
  neural networks: Orthogonal case.
\newblock \emph{Journal of Mathematical Physics}, 63\penalty0 (6), 2022.

\bibitem[Pastur \& Slavin(2023)Pastur and Slavin]{pastur2023random}
Pastur, L. and Slavin, V.
\newblock On random matrices arising in deep neural networks: General iid case.
\newblock \emph{Random Matrices: Theory and Applications}, 12\penalty0
  (01):\penalty0 2250046, 2023.

\bibitem[Pennington \& Worah(2017)Pennington and Worah]{pennington2017nonlinea}
Pennington, J. and Worah, P.
\newblock Nonlinear random matrix theory for deep learning.
\newblock \emph{Advances in neural information processing systems}, 30, 2017.

\bibitem[Ramzi et~al.(2021)Ramzi, Mannel, Bai, Starck, Ciuciu, and
  Moreau]{ramzi2021shine}
Ramzi, Z., Mannel, F., Bai, S., Starck, J.-L., Ciuciu, P., and Moreau, T.
\newblock Shine: Sharing the inverse estimate from the forward pass for
  bi-level optimization and implicit models.
\newblock \emph{arXiv preprint arXiv:2106.00553}, 2021.

\bibitem[Seddik et~al.(2020)Seddik, Louart, Tamaazousti, and
  Couillet]{seddik2020random}
Seddik, M. E.~A., Louart, C., Tamaazousti, M., and Couillet, R.
\newblock Random matrix theory proves that deep learning representations of
  gan-data behave as gaussian mixtures.
\newblock In \emph{International Conference on Machine Learning}, pp.\
  8573--8582. PMLR, 2020.

\bibitem[Truong(2023)]{truong2023global}
Truong, L.~V.
\newblock Global convergence rate of deep equilibrium models with general
  activations.
\newblock \emph{arXiv preprint arXiv:2302.05797}, 2023.

\bibitem[Tsuchida et~al.(2018)Tsuchida, Roosta, and
  Gallagher]{tsuchida2018invariance}
Tsuchida, R., Roosta, F., and Gallagher, M.
\newblock Invariance of weight distributions in rectified mlps.
\newblock In \emph{International Conference on Machine Learning}, pp.\
  4995--5004. PMLR, 2018.

\bibitem[Vershynin(2018)]{vershynin2018high}
Vershynin, R.
\newblock \emph{High-dimensional probability: An introduction with applications
  in data science}, volume~47.
\newblock Cambridge university press, 2018.

\bibitem[Winston \& Kolter(2020)Winston and Kolter]{winston2020monotone}
Winston, E. and Kolter, J.~Z.
\newblock Monotone operator equilibrium networks.
\newblock In \emph{Advances in Neural Information Processing Systems}, 2020.

\bibitem[Xiao et~al.(2017)Xiao, Rasul, and Vollgraf]{xiao2017fashion}
Xiao, H., Rasul, K., and Vollgraf, R.
\newblock Fashion-{{MNIST}}: A novel image dataset for benchmarking machine
  learning algorithms.
\newblock \emph{arXiv}, 2017.

\bibitem[Xie et~al.(2022)Xie, Wang, Ling, Li, Liu, and
  Lin]{xie2022optimization}
Xie, X., Wang, Q., Ling, Z., Li, X., Liu, G., and Lin, Z.
\newblock Optimization induced equilibrium networks: An explicit optimization
  perspective for understanding equilibrium models.
\newblock \emph{IEEE Transactions on Pattern Analysis and Machine
  Intelligence}, 2022.

\end{thebibliography}
\bibliographystyle{icml2024}

\appendix
\onecolumn
\begin{center}
  {\Large \textbf{Supplementary Material\\}} \vskip 0.1in \textbf{Deep Equilibrium Models are Almost Equivalent to Not-so-deep Explicit Models\\ for High-dimensional Gaussian Mixtures}
\end{center}

\section{Preliminaries}
\label{sec:appendix_pre}

 We are interested in the  associated conjugate kernel and the neural tangent kernel (Implicit-CK and Implicit-NTK, for short) of implicit neural networks defined in~\eqref{eq:deq}. According to the results in~\citep[Theorem 2]{feng2020neural}, the corresponding Implicit-CK takes the following form
\begin{equation}
\mG^* = \lim_{l\rightarrow\infty}\mG^{(l)},
\end{equation}
where the $(i,j)$-th entry of $\mG^{(l)}$ is defined recursively as $\mG^{(0)}_{ij} = 0$ and\footnote{ Note that the expectation is conditioned on the input data, and is taken with respect to the random weights.}
\begin{equation}
\mG^{(l)}_{ij} =\E[\phi(\ru^{(l)})\phi(\rv^{(l)})],\,\text{with}\, (\ru^{(l)},\rv^{(l)})\sim \gN\left(0, \begin{bmatrix}
	 \boldsymbol{\Lambda}_{ii}^{(l)} &  \boldsymbol{\Lambda}_{ij}^{(l)}
	\\
	\boldsymbol{\Lambda}_{ji}^{(l)} &  \boldsymbol{\Lambda}_{jj}^{(l)}
\end{bmatrix}\right), \quad l \geq 1,
\end{equation}
where $\boldsymbol{\Lambda}_{ij}^{(1)}=\vx_i^\top\vx_j$, and $\boldsymbol{\Lambda}_{ij}^{(l)}=\sigma_a^2\mG^{(l-1)}_{ij} + \sigma_b^2\vx_i^\top\vx_j$, \emph{i.e.}, $\boldsymbol{\Lambda}^{(l)} = \sigma_a^2\mG^{(l-1)} + \sigma_b^2\mX^\top\mX$.
The Implicit-NTK  is defined as $\mK^*=\lim_{l\rightarrow\infty}\mK^{(l)}$ whose the $(i,j)$-th entry  is defined as
\begin{equation}
\mK^{(l)}_{ij} = \sum_{h=1}^{l+1}\left(\mG^{(h-1)}_{ij}\prod_{h'=h}^{l+1}\dot\mG^{(h')}_{ij}\right),
\end{equation}
where $\dot\mG^{(l)}_{ij} = \sigma_a^2\E_{(\ru^{(l)},\rv^{(l)})}[\phi'(\ru^{(l)})\phi'(\rv^{(l)})]$.
The limit of Implicit-NTK is 
\begin{equation}
\mK^*_{ij} \equiv \frac{\mG^*_{ij}}{1-\dot\mG^*_{ij}}.
\end{equation}
We consider $n$ data vectors $\vx_1,\cdots,\vx_n\in\R^p$ independently drawn from one of the $K$-class Gaussian mixture $\mathcal{C}_1,\cdots,\mathcal{C}_K$ and denote $\mX=[\vx_1,\cdots,\vx_n]\in\R^{p\times n}$, with class $\mathcal{C}_a$ having cardinality $n_a$, \emph{i.e.}, for $\vx_i\in\mathcal{C}_a$, we have \[\vx_i\sim\gN(\vmu_a/\sqrt{p},\mC_a/p)\]
\begin{assum}~\label{assum:data}
	We assume that, as $n\rightarrow \infty$, we have, for $a\in\{1,\cdots,K\}$ that,
	\begin{itemize}
		\item $p/n \rightarrow c\in (0,\infty)$ and $n_a/n\rightarrow c_a \in (0,1)$; and
		\item $\|{\vmu}_a\|=\order{1}$; and
		\item for $\mC^{\circ}\equiv\sum_{a=1}^{K}\frac{n_a}{n}\mC_a$ and $\mC_a^\circ \equiv \mC_a-\mC^{\circ}$, we have $\|\mC_a\|=\order{1}$, $\tr \mC_a^\circ=\order{p^{\frac{1}{2}}}$ and $\tr(\mC_a\mC_b)=\order{p}$ for $a,b\in \{1,\cdots,K\}$; and
		\item  $\tau_0=\sqrt{\tr \mC^\circ/p}$ converges in $(0,\infty)$.
	\end{itemize}
\end{assum}
\paragraph{Some quantities.} We first introduce the following notations. For $\vx_i,\vx_j\in\R^p$ with $i\neq j$, let \[\vx_i=\vmu_i/\sqrt{p}+\beps_i/\sqrt{p},\quad \vx_j=\vmu_j/\sqrt{p}+\beps_j/\sqrt{p}, \] 
so that $\beps_i\sim\gN(0,\mC_i)$, $\beps_j\sim\gN(0,\mC_j)$, and 
\begin{align*}
	\vx_i^\top\vx_j &= \underbrace{\frac{1}{p}\beps_i^\top\beps_j}_{\order{p^{-1/2}}}+\underbrace{\frac{1}{p}\vmu_i^\top\vmu_j+\frac{1}{p}(\vmu_i^\top\beps_j+\vmu_j^\top\beps_i)}_{\order{p^{-1}}},\\
	\psi_i&=\frac{1}{p}\|\beps_i\|^2-\frac{1}{p}\tr\mC_i=\order{p^{-1/2}}, \quad s_i\equiv \|\mu_i\|^2/p+2\vmu_i^\top\beps_i/p=\order{p^{-1}},\\
	t_i &= \frac{1}{p}\tr \mC_i^\circ=\order{p^{-1/2}},\quad \tau_0 = \sqrt{\frac{1}{p}\mC^\circ}=\order{1},\\
	\chi_i &= \underbrace{t_i+\psi_i}_{\order{p^{-1/2}}}+\underbrace{s_i}_{\order{p^{-1}}}=\|\vx_i\|^2-\tau_0^2.
\end{align*}
It can be checked that
\begin{align*}
	\|\vx_i\|^2&=\frac{1}{p}(\vmu_i+\beps_i)^\top(\vmu_i+\beps_i)=\frac{1}{p}\|\vmu_i\|^2+\frac{2}{p}\vmu_i^\top\beps_i+\frac{1}{p}\beps_i^\top\beps_i\\
	&=\underbrace{\frac{1}{p}\|\vmu_i\|^2+\frac{2}{p}\vmu_i^\top\beps_i}_{\equiv s_i=\order{p^{-1}}}+\underbrace{\frac{1}{p}\tr \mC^\circ}_{\equiv\tau_0^2=\order{1}} + \underbrace{\frac{1}{p} \tr\mC_i^\circ}_{\equiv t_i=\order{p^{-1/2}}} +\underbrace{\psi_i}_{\order{p^{-1/2}}}
\end{align*}
By Taylor-expanding $\sqrt{\|\vx_i\|^2}$ around $\tau_0^2$, we have 
\begin{equation}
	\|\vx_i\| = \tau_0 + \frac{1}{2\tau_0}(\|\vmu_i\|^2/p+2\vmu_i^\top\beps_i/p+t_i+\psi_i)-\frac{1}{8\tau_0^3}(t_i+\psi_i)^2+\order{p^{-3/2}}.
	\label{eq:xnorm}
\end{equation}
Additionally, we denote $S_{ij}$ terms of the form 
\begin{align*}
	S_{ij}\equiv S_{ij}(\gamma_1,\gamma_2) = \frac{1}{p}\beps_i^\top\beps_j(\gamma_1(t_i+\psi_i)+\gamma_2(t_j+\psi_j)),
\end{align*}
for random or deterministic scalars $\gamma_1,\gamma_2=\order{1}$ (with high probability when being random). Note that $S_{ij}=\order{p^{-1}}$ and more importantly, it leads to, in matrix form, a matrix of spectral norm order $\order{p^{-1}}$~\cite{couillet2016kernel}.

Moreover, we recursively define $\tau_l$ as
\begin{equation}
	\tau_l =  \sqrt{\sigma_a^2\E[\phi^2(\tau_{l-1}\xi)]+\sigma_b^2\tau_0^2},
	\label{eq:taul}
\end{equation}
for $l=1,2,\cdots$. The following lemma shows that the unique fixed point of \eqref{eq:taul} exits under Assumption~\ref{cond:tau^*}.
\begin{lem}
	Let Assumption~\ref{cond:tau^*} hold. As $l\rightarrow \infty$, $\tau_l$ converges to a fixed point $\tau^*$ such that
	\begin{equation*}
		\lim_{l\rightarrow\infty}\tau_l\equiv \tau_* =  \sqrt{\sigma_a^2\E[\phi^2(\tau_*\xi)]+\sigma_b^2\tau_0^2}.
		\label{lem:tau*}
	\end{equation*}
\end{lem}
\begin{proof}
	Let $t = \tau_{l-1}^2$. By taking the derivative with respect to $t$ on the RHS of~\eqref{eq:tau*}, we have 
	\begin{align*}
		&\frac{\partial }{\partial t}\left(\sigma_a^2\E\left[f(\tau_{l-1}\xi)\right]+\sigma_b^2\tau_0^2\right) \\
		=&\sigma_a^2\frac{\partial }{\partial t}\E\left[f(\sqrt{t}\cdot\xi)\right]\\
		=&\sigma_a^2\frac{\partial }{\partial t}\left(\int   \frac{1}{\sqrt{2\pi}}f(\sqrt{t}\cdot x) e^{-\frac{x^2}{2}}dx\right)\\
		=&\sigma_a^2 \frac{1}{\sqrt{2\pi}}\int f'(\sqrt{t}\cdot x) 
		\frac{x}{2\sqrt{t}}e^{-\frac{x^2}{2}}dx\\
		=&\frac{\sigma_a^2}{2}\cdot\E[f''(\tau_{l-1}\xi)], \quad \text{by the Gaussian integration by parts formula,}
	\end{align*}
	which implies that the RHS of~\eqref{eq:tau*} is a \emph{contractive mapping} if
	\begin{align*}
		\sigma_a^2<\frac{2}{L_2}.
	\end{align*}
	As a result, under Assumption~\ref{cond:tau^*}, the unique fixed point $\tau_*$ exists. 
\end{proof}
The quantity $\tau_*$ will play a crucial role in our proof.
\section{Proof of Theorem~\ref{thm:imCK}}\label{sec:proof_imCK}
With loss of generality, we assume that $\mG^{(0)}  = \E[\phi^2(\tau_*\xi)]\cdot\mI_n$, \emph{i.e.},  $\mG_{ii}^{(0)}= \E[\phi^2(\tau_*\xi)]$ and $\mG_{ij}^{(0)} =0$ for $i\neq j$. 

We prove Theorem~\ref{thm:imCK} by performing  induction on the hypothesis that 
$\|\mG^{(l-1)}-\tmG^{(l-1)}\|\rightarrow 0$ holds at layer $l-1$ with 
\begin{equation}
	\tmG^{(l-1)}\equiv \alpha_{l-1,1}\mX^\top\mX+\mV\mC^{(l-1)}\mV^\top + (\E[\phi^2(\tau_{*}\xi)]-\tau_0^2\alpha_{l-1,1})\mI_n,
\label{eq:hypothesisonG}
\end{equation}
for $\mC^{(l-1)}= \left[\begin{array}{cc}
	\alpha_{l-1,2}\vt\vt^\top + \alpha_{l-1,3}\mT &\alpha_{l-1,2}\vt  \\
	\alpha_{l-1,2}\vt^\top &\alpha_{l-1,2} 
\end{array}\right]$, and  work on  $\mG^{(l)}$ at layer $l$.

Note that $\boldsymbol{\Lambda}_{ij}^{(l)}=\sigma_a^2\mG^{(l-1)}_{ij} + \sigma_b^2\vx_i^\top\vx_j$, \emph{i.e.}, $\boldsymbol{\Lambda}^{(l)} = \sigma_a^2\mG^{(l-1)} + \sigma_b^2\mX^\top\mX$. Thus, it holds that
$\|\bLambda^{(l)}-\tbLambda^{(l)}\|\rightarrow 0$ for
\begin{align*}
	\tbLambda^{(l)}\equiv \lambda_{l,1}\mX^\top\mX+\mV\mC_{\bLambda}^{(l)}\mV^\top + (\tau_{*}^2-\tau_0^2\lambda_{l,1})\mI_n,
\end{align*}
for $\mC_{\bLambda}^{(l)}= \left[\begin{array}{cc}
	\lambda_{l,2}\vt\vt^\top + \lambda_{l,3}\mT &\lambda_{l,2}\vt  \\
	\lambda_{l,2}\vt^\top &\lambda_{l,2} 
\end{array}\right]$ where $ \lambda_{l,1} = \sigma_a^2\alpha_{l-1,1} + \sigma_b^2$, $ \lambda_{l,2} = \sigma_a^2\alpha_{l-1,2}$, and $ \lambda_{l,2} = \sigma_a^2\alpha_{l-1,3}$.

The following lemma plays an important role in our proof.
\begin{lem}[~\cite{du2022lossless}]
	Let Assumptions~\ref{assum:activation}-\ref{assum:data} hold. Given a matrix $\bLambda\in\R^{n\times n}$ such that
	\begin{align*}
		\bLambda_{ii} &= \tau^2+\lambda_{4} \chi_i+ \lambda_{5}(t_i+\psi_i)^2+\order{p^{-3/2}}\\
		\bLambda_{ij} &= \lambda_{1}\vx_i^\top\vx_j + \lambda_{2}(t_i+\psi_i)(t_j+\psi_j) + \lambda_{3}\left(\frac{1}{p}\beps_i^\top\beps_j\right)^2 + S_{ij} + \order{p^{-3/2}},
	\end{align*}
	where $\lambda_k$, $k=1, \cdots, 5$, and $\tau$ are arbitrary constants, it holds that 
	\begin{align*}
		&\E\left[\phi\left(\sqrt{\bLambda_{ii}}\cdot \xi_i\right)\times \phi\left(\frac{\bLambda_{ij}}{\sqrt{\bLambda_{ii}}}\cdot\xi_i + \sqrt{\bLambda_{jj}-\frac{\left(\bLambda_{ij}\right)^2}{\bLambda_{ii}}}\cdot\xi_j\right)\right] \\
		=&\E[\phi'(\tau\xi)]^2 + \E[\phi''(\tau\xi)]^2\cdot\lambda_{1}\vx_i^\top\vx_j\\
		&+\E[\phi'(\tau\xi)]\E[\phi'''(\tau\xi)]\cdot\frac{\lambda_{4}}{2}\left(\chi_i+\chi_j\right)+\order{p^{-1}},
	\end{align*}
   for independent  $\xi_i$, $\xi_j$ and $\xi\sim \gN(0,1)$.
	Moreover, if the activation function $\phi(\cdot)$ is ``centered", such that $\E[\phi(\tau\xi)]=0$, it holds that
	\begin{align*}
		&\E\left[\phi\left(\sqrt{\bLambda_{ii}}\cdot \xi_i\right)\times \phi\left(\frac{\bLambda_{ij}}{\sqrt{\bLambda_{ii}}}\cdot\xi_i + \sqrt{\bLambda_{jj}-\frac{\left(\bLambda_{ij}\right)^2}{\bLambda_{ii}}}\cdot\xi_j\right)\right] \\
		=& \E[\phi'(\tau\xi)]^2\bLambda_{ij} + \frac{\lambda_{1}^2}{2}\E[\phi''(\tau\xi)]^2\left(\frac{1}{p}\beps_i^\top\beps_j\right)^2 + \frac{\lambda_{4}^2}{4}\E[\phi''(\tau\xi)]^2(t_i+\psi_i)(t_j+\psi_j) \\
		&+ S_{ij} +\order{p^{-3/2}}.
	\end{align*}
	\label{lem:liao}
\end{lem}

\paragraph{On the diagonal.}
By induction hypothesis on the layer $l$, we have 
\begin{equation}
	\bLambda_{ii}^{(l)} = \tau_{*}^2+\lambda_{l,4} \chi_i+ \lambda_{l,5} (t_i+\psi_i)^2+\order{p^{-3/2}}.
\end{equation}
For $l=1$, $\bLambda^{(1)} = \sigma_a^2\mG_{ii}^{(0)}+\sigma_b^2\|\vx_i\|^2=\sigma_a^2\E[\phi^2(\tau^*\xi)]+\sigma_b^2\|\vx_i\|^2=\tau_*^2$, and the hypothesis holds with $\lambda_{1,4}=\lambda_{1,5}=0$.

\noindent For $l>1$, 
by~\eqref{eq:imck}, we have 
\begin{align*}
	\bLambda_{ii}^{(l+1)}=\sigma_a^2\E\left[\phi\left(\sqrt{\bLambda_{ii}^{(l)}}\cdot\xi\right)^2\right]+(1-\sigma_a^2)\|\vx_i\|^2,
\end{align*}
for $\xi\sim\gN(0,1)$.

By Taylor-expanding, one gets 
\begin{align*}
	\sqrt{\bLambda_{ii}^{(l)} }=\tau_{*}+\frac{1}{2\tau_{*}}\lambda_{l,4}\chi_i + \frac{4\tau^2_{*}\lambda_{l,5}-\lambda_{l,4}^2}{8\tau_{*}^3}(t_i+\psi_i)^2+\order{p^{-3/2}}.
\end{align*}
For simplicity, we denote the shortcut $f(\cdot)=\phi^2(\cdot)$. By Talor-expanding and~\eqref{eq:xnorm}, one gets 
\begin{equation}
	\begin{split}
		\bLambda_{ii}^{(l+1)} =& \sigma_a^2\E\left[\phi\left(\sqrt{\bLambda_{ii}^{(l)}}\cdot\xi\right)^2\right]+\sigma_b^2\|\vx_i\|^2=\sigma_a^2\E\left[f\left(\sqrt{\bLambda_{ii}^{(l)}}\cdot\xi\right)\right]+\sigma_b^2\|\vx_i\|^2\\
		=&\sigma_a^2\E\left[f(\tau_{*}\xi) + f'(\tau_{*}\xi)\xi\left(\frac{1}{2\tau_{*}}\lambda_{l,4}\chi_i + \frac{4\tau^2_{*}\lambda_{l,5}-\lambda_{l,4}^2}{8\tau_{*}^3}(t_i+\psi_i)^2\right)\right]\\
		&+\sigma_a^2\E\left[\frac{1}{2}f''(\tau_{*}\xi)\xi^2\right]\frac{\lambda_{l,4}^2}{4\tau_{*}}(t_i+\psi_i)^2+ \sigma_b^2(\tau_0^2+\chi_i)+\order{p^{-3/2}}\\
		=& \sigma_a^2\E\left[f(\tau_{*}\xi)\right]+\sigma_b^2\tau_0^2 + \left(\sigma_a^2\frac{\lambda_{l,4}}{2}\E[f''(\tau_{*}\xi)] + \sigma_b^2\right)\chi_i \\
		&+ \sigma_a^2\frac{4\lambda_{l,5}\E[f''(\tau_{*}\xi)]+\lambda_{l,4}^2\E[f''''(\tau_{*}\xi)]}{8}(t_i+\psi_i)^2+\order{p^{-3/2}},
		\label{eq:Gii}    
	\end{split}
\end{equation}
where we use the facts that
\begin{align*}
	\E[f'(\tau_{*}\xi)]=\tau_{*}\E[f''(\tau_{*}\xi)],\quad \E[f''''(\tau_{*}\xi)(\xi^2-1)]=\tau_{*}^2\E[f''''(\tau_{*}\xi)],
\end{align*}
for $\xi\sim\gN(0,1)$, as a result of the Gaussian integration by parts formula.

Thus , we prove that $\bLambda_{ii}^{(l+1)} = \tau_{*}^2+\lambda_{l+1,4} \chi_i+ \lambda_{l+1,5} (t_i+\psi_i)^2+\order{p^{-3/2}}$, where
\begin{equation}
	\begin{split}
		\lambda_{l+1,4} &= \frac{\sigma_a^2}{2}\E[f''(\tau_{*}\xi)]\lambda_{l,4} + \sigma_b^2, \\
		\lambda_{l+1,5} &= \frac{\sigma_a^2}{2}\E[f''(\tau_{*}\xi)] \lambda_{l,5}+\frac{\sigma_a^2}{8}\E[f''''(\tau_{*}\xi)]\lambda_{l,4}^2.
	\end{split}
	\label{eq:ckrelation}
\end{equation}

By Lemma~\ref{lem:tau*}, under Assumption~\ref{cond:tau^*}, it holds that $\frac{\sigma_a^2}{2}\E[f''(\tau_{l-1}\xi)] <1$,  which implies that, as $l\rightarrow \infty$,  the iterations in~\eqref{eq:ckrelation} converge. Let $l\rightarrow \infty$, we obtain that
\begin{equation}
	\begin{split}
		\lambda_{*,4} &\equiv\lim_{l\rightarrow\infty}\lambda_{l,4}=\left(1-\frac{\sigma_a^2}{2}\E[f''(\tau_*\xi)]\right)^{-1}\sigma_b^2\\
		\lambda_{*,5} &\equiv\lim_{l\rightarrow\infty} \lambda_{l,5}= \frac{\sigma_a^2}{8} \left(1-\frac{\sigma_a^2}{2}\E[f''(\tau_*\xi)]\right)^{-1}\E[f''''(\tau^*\xi)]\lambda_{*,4}^2.
	\end{split}
\end{equation}

\paragraph{Off the diagonal.} 
For $i\neq j$, by induction hypothesis on the layer $l-1$, we have 
\begin{align*}
	\bLambda _{ij}^{(l)} = \lambda_{l,1}A_{ij} + \lambda_{l,2}(t_i+\psi_i)(t_j+\psi_j) + \lambda_{l,3}\left(\frac{1}{p}\beps_i^\top\beps_j\right)^2 + S_{ij} + \order{p^{-3/2}}.
\end{align*}
Using the Gram-Schmidt orthogonalization for standard Gaussian random variable, we write 
\begin{align*}
	\bLambda_{ii}^{(l+1)} &= \sigma_a^2\E\left[\phi^2\left(\sqrt{\bLambda_{ii}^{(l)}}\cdot \xi_i\right)\right]+\sigma_b^2\|\vx_i\|^2,\\
	\bLambda_{ij}^{(l+1)} &= \sigma_a^2\E\left[\phi\left(\sqrt{\bLambda_{ii}^{(l)}}\cdot \xi_i\right)\times \phi\left(\frac{\bLambda_{ij}^{(l)}}{\sqrt{\bLambda_{ii}^{(l)}}}\cdot\xi_i + \sqrt{\bLambda_{jj}^{(l)}-\frac{\left(\bLambda_{ij}^{(l)}\right)^2}{\bLambda_{ii}^{(l)}}}\cdot\xi_j\right)\right] + \sigma_b^2\vx_i^\top\vx_j.
\end{align*}

Using Lemma~\ref{lem:liao}, we have 
\begin{align*}
	\bLambda_{ij}^{(l+1)} 
	=& \sigma_a^2\E\left[\phi\left(\sqrt{\bLambda_{ii}^{(l)}}\cdot \xi_i\right)\times \phi\left(\frac{\bLambda_{ij}^{(l)}}{\sqrt{\bLambda_{ii}^{(l)}}}\cdot\xi_i + \sqrt{\bLambda_{jj}^{(l)}-\frac{\left(\bLambda_{ij}^{(l)}\right)^2}{\bLambda_{ii}^{(l)}}}\cdot\xi_j\right)\right] + \sigma_b^2\vx_i^\top\vx_j\\
	=&\sigma_a^2 \E[\phi'(\tau_{*}\xi)]^2\bLambda_{ij}^{(l)} \\
	&+ \sigma_a^2\left(\frac{\lambda_{l,1}}{2}\E[\phi''(\tau_{*}\xi)]^2\left(\frac{1}{p}\beps_i^\top\beps_j\right)^2 + \frac{\lambda_{l,4}^2}{4}\E[\phi''(\tau_{*}\xi)]^2(t_i+\psi_i)(t_j+\psi_j)\right)\\
	&+ S_{ij} + \sigma_b^2\vx_i^\top\vx_j+\order{p^{-3/2}}\\
	=&\sigma_a^2 \E[\phi'(\tau_{*}\xi)]^2\left(\lambda_{l,1}\vx_i^\top\vx_j + \lambda_{l,2}(t_i+\psi_i)(t_j+\psi_j) + \lambda_{l,3}\left(\frac{1}{p}\beps_i^\top\beps_j\right)^2\right) \\
	&+ \sigma_a^2\left(\frac{\lambda_{l,1}^2}{2}\E[\phi''(\tau_{*}\xi)]^2\left(\frac{1}{p}\beps_i^\top\beps_j\right)^2 + \frac{\lambda_{l,4}^2}{4}\E[\phi''(\tau_{*}\xi)]^2(t_i+\psi_i)(t_j+\psi_j)\right)\\
	&+ S_{ij} +\sigma_b^2\vx_i^\top\vx_j+\order{p^{-3/2}}.
\end{align*}
Consequently, it holds that
\begin{equation}
	\begin{split}
		\bLambda_{ij}^{(l+1)} = \lambda_{l+1,1}\vx_i^\top\vx_j + \lambda_{l+1,2}(t_i+\psi_i)(t_j+\psi_j) + \lambda_{l+1,3}\left(\frac{1}{p}\beps_i^\top\beps_j\right)^2 + S_{ij} +\order{p^{-3/2}},
	\end{split}
	\label{eq:Gij}
\end{equation}
where 
\begin{equation}
	\begin{split}
		\lambda_{l+1,1} &= \sigma_a^2 \E[\phi'(\tau_{*}\xi)]^2\lambda_{l,1} + \sigma_b^2,\\
		\lambda_{l+1,2} &= \sigma_a^2 \E[\phi'(\tau_{*}\xi)]^2\lambda_{l,2} + \frac{\sigma_a^2 }{4}\E[\phi''(\tau_{*}\xi)]^2\lambda_{l,4}^2,\\
		\lambda_{l+1,3} &= \sigma_a^2 \E[\phi'(\tau_{*}\xi)]^2\lambda_{l,3} + \frac{\sigma_a^2}{2}\E[\phi''(\tau_{*}\xi)]^2\lambda_{l,1}^2.
	\end{split}
	\label{eq:lambda}
\end{equation}
\paragraph{Assembling in matrix form.} By using the fact that $\|\mM\|_2\leq n\max_{i,j}|\mM_{ij}|$ for $\mM\in\R^{n\times n}$ and $\{S_{ij}\}_{ij}=\mathcal{O_{\|\cdot\|}}(p^{-1/2})$~\cite{couillet2016kernel}, and by noting the fact that $\boldsymbol{\Lambda}^{(l+1)} = \sigma_a^2\mG^{(l)} + \sigma_b^2\mX^\top\mX$, \emph{i.e}, $ \lambda_{l,1} = \sigma_a^2\alpha_{l-1,1} + \sigma_b^2$, $ \lambda_{l,2} = \sigma_a^2\alpha_{l-1,2}$, and $ \lambda_{l,2} = \sigma_a^2\alpha_{l-1,3}$,
it holds that 
\begin{equation}
	\mG^{(l)} = \alpha_{l,1}\mX^\top\mX + \mV\mC^{(l)}\mV^\top+ (\E[\phi^2(\tau_*\xi)]-\tau_0^2\alpha_{l,1})\mI_n+\mathcal{O_{\|\cdot\|}}(p^{-\frac{1}{2}})
\end{equation}
where  
\begin{equation}
	\mV = \left[\mJ/\sqrt{p}, \boldsymbol{\psi}\right], \quad \mC^{(l)} = \left[\begin{array}{cc}
		\alpha_{l,2}\vt\vt^\top +\alpha_{l,3}\mT& \alpha_{l,2}\vt  \\
		\alpha_{l,2}\vt^\top & \alpha_{l,2}
	\end{array}\right],
\end{equation}
with  non-negative scalars $\alpha_{l,1}, \alpha_{l,2}, \alpha_{l,3}, \alpha_{l,4}\geq 0$ defined recursively as 
\begin{equation}
		\begin{split}
		\alpha_{l,1} &= \sigma_a^2 \E[\phi'(\tau_{*}\xi)]^2\alpha_{l-1,1} + \sigma_b^2\E[\phi'(\tau_{*}\xi)]^2,\\
		\alpha_{l,2} &= \sigma_a^2 \E[\phi'(\tau_{*}\xi)]^2\alpha_{l-1,2} + \frac{1 }{4}\E[\phi''(\tau_{*}\xi)]^2\alpha_{l-1,4}^2,\\
		\alpha_{l,3} &= \sigma_a^2 \E[\phi'(\tau_{*}\xi)]^2\alpha_{l-1,3} + \frac{1}{2}\E[\phi''(\tau_{*}\xi)]^2(\sigma_a^2\alpha_{l-1,1}+\sigma_b^2)^2,\\
		\alpha_{l,4} &= \frac{\sigma_a^2}{2}\E[(\phi^2(\tau_*\xi))'']\alpha_{l-1,4} + \sigma_b^2, 
	\end{split}
\label{eq:alphal}
\end{equation}
Note that it holds that $\sigma_a^2 \E[\phi'(\tau_*\xi)]^2<1$  and $\frac{1}{2}\sigma_a^2 \E[(\phi^2(\tau_*\xi))'']<1$ under Assumptions~\ref{cond:G*} and~\ref{cond:tau^*}. This means that, as $l\rightarrow\infty$, the iterations in~\eqref{eq:alphal} converge. Let $l \rightarrow \infty$, we obtain that
\begin{equation}
	\mG^* = \alpha_{*,1}\mX^\top\mX + \mV\mC\mV^\top+ (\E[\phi^2(\tau_*\xi)]-\tau_0^2\alpha_{*,1})\mI_n+\mathcal{O_{\|\cdot\|}}(p^{-\frac{1}{2}})
\end{equation}
where 
\begin{equation}
	\mV = \left[\mJ/\sqrt{p}, \boldsymbol{\psi}\right], \quad \mC = \left[\begin{array}{cc}
		\alpha_{*,2}\vt\vt^\top +\alpha_{*,3}\mT& \alpha_{*,2}\vt  \\
		\alpha_{*,2}\vt^\top & \alpha_{*,2}
	\end{array}\right],
\end{equation}
with  non-negative scalars $\alpha_{*,1}, \alpha_{*,2}, \alpha_{*,3}, \alpha_{*,4}\geq 0$ defined as 
\begin{align}
	\alpha_{*,1} &= \frac{\sigma_b^2\E[\phi'(\tau_{*}\xi)]^2}{1-\sigma_a^2 \E[\phi'(\tau_{*}\xi)]^2},
	&\alpha_{*,2} &= \frac{\alpha_{*,4}^2\E[\phi''(\tau_{*}\xi)]^2}{4(1-\sigma_a^2 \E[\phi'(\tau_{*}\xi)]^2)},\\
	\alpha_{*,3} &= \frac{(\sigma_a^2\alpha_{*,1}+\sigma_b^2)^2\E[\phi''(\tau_{*}\xi)]^2}{2(1-\sigma_a^2 \E[\phi'(\tau_{*}\xi)]^2)},
	&\alpha_{*,4} &= \frac{\sigma_b^2}{1-\frac{\sigma_a^2}{2}\E[(\phi^2(\tau_*\xi))'']}.
\end{align}

\section{Proof of Theorem~\ref{thm:imNTK}}\label{sec:proof_imNTK}
\subsection{ The CK \text{$\dot\mG$}}
Before proving Theorem~\ref{thm:imNTK}, one needs to deal with the CK $\dot\mG$.

Recall that 
\begin{align*}
	\dot\mG^{(l)}_{ij} = \sigma_a^2\E_{(\ru^{(l)},\rv^{(l)})}[\phi'(\ru^{(l)})\phi'(\rv^{(l)})], \,\text{with}\, (\ru^{(l)},\rv^{(l)})\sim \gN\left(0, \begin{bmatrix}
		\boldsymbol{\Lambda}_{ii}^{(l)} &  \boldsymbol{\Lambda}_{ij}^{(l)}
		\\
		\boldsymbol{\Lambda}_{ji}^{(l)} &  \boldsymbol{\Lambda}_{jj}^{(l)}
	\end{bmatrix}\right).
\end{align*}
Using the Gram-Schmidt orthogoalizaiton procedure, we have 
\begin{equation}
	\begin{split}
		\dot\mG_{ii}^{(l)} &= \sigma_a^2\E\left[\phi'\left(\sqrt{\bLambda_{ii}^{(l)}}\cdot \xi_i\right)^2\right]\\
		\dot\mG_{ij}^{(l)} & = \sigma_a^2\E\left[\phi'\left(\sqrt{\bLambda_{ii}^{(l)}}\cdot \xi_i\right)\times \phi'\left(\frac{\bLambda_{ij}^{(l)}}{\sqrt{\bLambda_{ii}^{(l)}}}\cdot\xi_i + \sqrt{\bLambda_{jj}^{(l)}-\frac{\left(\bLambda_{ij}^{(l)}\right)^2}{\bLambda_{ii}^{(l)}}}\cdot\xi_j\right)\right]     
	\end{split}
\end{equation}
\paragraph{On the diagonal.}
First, recall that
\begin{align*}
	\sqrt{\bLambda_{ii}^{(l)} }=\tau_{*}+\frac{1}{2\tau_{*}}\lambda_{l,4}\chi_i + \frac{4\tau^2_{*}\lambda_{l,5}-\lambda_{l,4}^2}{8\tau_{*}^3}(t_i+\psi_i)^2+\order{p^{-3/2}}.
\end{align*}
Denote the shortcut $f(t) =(\phi'(t))^2$ for simplicity, using Taylor-expand again, we have
\begin{align*}
	\dot\mG_{ii}^{(l)} =& \sigma_a^2\E\left[\phi'\left(\sqrt{\bLambda_{ii}^{(l)}}\cdot\xi\right)^2\right]=\sigma_a^2\E\left[f\left(\sqrt{\bLambda_{ii}^{(l)}}\cdot\xi\right)\right]\\
	=&\sigma_a^2\E\left[f(\tau_{*}\xi) + f'(\tau_{*}\xi)\xi\left(\frac{1}{2\tau_{*}}\lambda_{l,4}\chi_i + \frac{4\tau^2_{*}\lambda_{l,5}-\lambda_{l,4}^2}{8\tau_{l-1}^3}(t_i+\psi_i)^2\right)\right]\\
	&+\sigma_a^2\E\left[\frac{1}{2}f''(\tau_{*}\xi)\xi^2\right]\frac{\lambda_{l,4}^2}{4\tau_{*}^2}(t_i+\psi_i)^2+\order{p^{-3/2}}\\
	=& \sigma_a^2\E\left[f(\tau_{*}\xi)\right] + \left(\sigma_a^2\frac{\lambda_{l,4}}{2}\E[f''(\tau_{*}\xi)] \right)\chi_i \\
	&+ \sigma_a^2\frac{4\lambda_{l,5}\E[f''(\tau_{*}\xi)]+\lambda_{l,4}^2\E[f''''(\tau_{*}\xi)]}{8}+\order{p^{-3/2}},
\end{align*}
Thus, we conclude that 
\begin{equation}
	\dot\mG^{(l)}_{ii} =\sigma_a^2\E\left[\phi'\left(\sqrt{\bLambda_{ii}^{(l)}}\xi\right)^2\right]= \sigma_a^2\dot\tau_*^2 + \order{p^{-1/2}},
	\label{eq:dotGii}
\end{equation}
with the sequence $\dot\tau_*$ defined as follows
\begin{align*}
	\dot\tau_* =\sqrt{\E\left[\phi'(\tau_*\xi)^2\right]}.
\end{align*}
\paragraph{Off the diagonal.}
For $i\neq j$, by Lemma~\ref{lem:liao}, it holds that
\begin{align*}
	&\E\left[\phi'\left(\sqrt{\bLambda_{ii}^{(l)}}\cdot \xi_i\right)\times \phi'\left(\frac{\bLambda_{ij}^{(l)}}{\sqrt{\bLambda_{ii}^{(l-1)}}}\cdot\xi_i + \sqrt{\bLambda_{jj}^{(l)}-\frac{\left(\bLambda_{ij}^{(l)}\right)^2}{\bLambda_{ii}^{(l)}}}\cdot\xi_j\right)\right] \\
	=&\E[\phi'(\tau_{*}\xi)]^2 + \E[\phi''(\tau_{*}\xi)]^2\cdot\lambda_{l,1}\vx_i^\top\vx_j\\
	&+E[\phi'(\tau_{*}\xi)]E[\phi'''(\tau_{*}\xi)]\cdot\frac{\lambda_{l,4}}{2}\left(\chi_i+\chi_j\right)+\order{p^{-1}}
\end{align*}
Thus, we have 
\begin{equation}
	\begin{split}
		\dot\mG_{ij}^{(l)} & = \sigma_a^2\E\left[\phi'\left(\sqrt{\bLambda_{ii}^{(l)}}\cdot \xi_i\right)\times \phi'\left(\frac{\bLambda_{ij}^{(l)}}{\sqrt{\bLambda_{ii}^{(l)}}}\cdot\xi_i + \sqrt{\bLambda_{jj}^{(l)}-\frac{\left(\bLambda_{ij}^{(l)}\right)^2}{\bLambda_{ii}^{(l)}}}\cdot\xi_j\right)\right]  \\
		&=\sigma_a^2\E[\phi'(\tau_{*}\xi)]^2 + \sigma_a^2\E[\phi''(\tau_{*}\xi)]^2\cdot\lambda_{l,1}\vx_i^\top\vx_j\\
		&+\sigma_a^2\E[\phi'(\tau_{*}\xi)]\E[\phi'''(\tau_{*}\xi)]\cdot\frac{\lambda_{l,4}}{2}\left(\chi_i+\chi_j\right)
		+\order{p^{-1}}\\
		&=\dot\alpha_{l,0} +  \dot\alpha_{l,1}\vx_i^\top\vx_j+ \dot\alpha_{l,2}\left(\chi_i+\chi_j\right)+\order{p^{-1}},
		\label{eq:dotGij}
	\end{split}
\end{equation}
with 
\begin{align*}
	\dot\alpha_{l,0} &= \sigma_a^2\E[\phi'(\tau_{*}\xi)]^2,\\
	\dot\alpha_{l,1} &= \sigma_a^2\E[\phi''(\tau_{*}\xi)]^2\lambda_{l,1} = \sigma_a^2\E[\phi''(\tau_{*}\xi)]^2(\sigma_a^2\alpha_{l-1,1}+\sigma_b^2),\\
	\dot\alpha_{l,2} &= \frac{\sigma_a^2}{2}\E[\phi'(\tau_{*}\xi)]\E[\phi'''(\tau_{*}\xi)]\lambda_{l,4}=\frac{\sigma_a^2}{2}\E[\phi'(\tau_{*}\xi)]\E[\phi'''(\tau_{*}\xi)]\alpha_{l,4}.
\end{align*}
As $l\rightarrow\infty$, it holds that 
\begin{align*}
	\dot\alpha_{*,0} &\equiv\lim_{l\rightarrow\infty}\dot\alpha_{l,0}= \sigma_a^2\E[\phi'(\tau_{*}\xi)]^2,\\
	\dot\alpha_{*,1} &\equiv\lim_{l\rightarrow\infty}\dot\alpha_{l,1}= \sigma_a^2\E[\phi''(\tau_{*}\xi)]^2(\sigma_a^2\alpha_{*,1}+\sigma_b^2),\\
	\dot\alpha_{*,2} &\equiv\lim_{l\rightarrow\infty}\dot\alpha_{l,2}= \frac{\sigma_a^2}{2}\E[\phi'(\tau_{*}\xi)]\E[\phi'''(\tau_{*}\xi)]\alpha_{*,4}.
\end{align*}
\subsection{Implicit NTKs}
With the above results at hand, we now proceed to the proof of  Theorem~\ref{thm:imNTK}.
We assume the induction hypothesis holds for $l-1$, that
\begin{align*}
	\mK_{ii}^{(l-1)} &= \kappa_{l-1}^2 + \order{p^{-1/2}},\\
	\mK_{ij}^{(l-1)} &= \beta_{l-1,1}\vx_i^\top\vx_j + \beta_{l-1,2}(t_i+\psi_i)(t_i+\psi_i) +\beta_{l-1,3}\left(\frac{1}{p}\beps_i^\top\beps_j\right)^2+S_{ij}+\order{p^{-3/2}}.
\end{align*}
\paragraph{On the diagonal.} First, using the results of~\eqref{eq:Gii} and~\eqref{eq:dotGii}, we have 
\begin{align*}
	\mK_{ii}^{(l)} = \mG_{ii}^{(l)} + \mK_{ii}^{(l-1)}\cdot\dot\mG_{ii}^{(l)} = \E[\phi^2(\tau_*\xi)]+ \sigma_a^2\kappa_{l-1}^2 \E[\phi'(\tau_*\xi)^2]+\order{p^{-1/2}}.
\end{align*}
Thus, it holds that 
\begin{align*}
	\mK_{ii}^{(l)} = \kappa_l^2 + \order{p^{-1/2}},
\end{align*}
with
\begin{align*}
	\kappa_l^2 = \E\left[\phi^2(\tau_*\xi)\right] +\sigma_a^2 \E[\phi'(\tau_*\xi)^2]\cdot \kappa_{l-1}^2.
\end{align*}
Under~\Cref{cond:G*}, it holds that $\sigma_a^2\E\left[\phi'(\tau_*\xi)^2\right]<1$. Thus, for $l\rightarrow \infty$, one gets that
\begin{equation}
	\kappa_*^2 \equiv\lim_{l\rightarrow\infty}\kappa_l^2= \left(1-\sigma_a^2\E\left[\phi'(\tau_*\xi)^2\right]\right)^{-1}\E\left[\phi^2(\tau_*\xi)\right].
\end{equation}
\paragraph{Off the diagonal.} 
For $i\neq j$, using the results of~\eqref{eq:Gij} and ~\eqref{eq:Gij},  we get
\begin{align*}
	\mK_{ij}^{(l)} =& \mG_{ij}^{(l)} + \mK_{ij}^{(l-1)}\dot\mG_{ij}^{(l)}\\
	=& \alpha_{l,1}\vx_i^\top\vx_j + \alpha_{l,2}(t_i+\psi_i)(t_j+\psi_j) + \alpha_{l,3}\left(\frac{1}{p}\beps_i^\top\beps_j\right)^2\\
	&+\left(\beta_{l-1,1}\vx_i^\top\vx_j + \beta_{l-1,2}(t_i+\psi_i)(t_j+\psi_j) +\beta_{l-1,3}\left(\frac{1}{p}\beps_i^\top\beps_j\right)^2\right)\\
	&\times \left(\dot\alpha_{l,0} +  \dot\alpha_{l,1}\vx_i^\top\vx_j+ \dot\alpha_{l,2}\left(\chi_i+\chi_j\right)+\order{p^{-1}}\right) + \order{p^{-3/2}}\\
	=& (\alpha_{l,1}+\beta_{l-1,1}\cdot\dot\alpha_{l,0})\vx_i^\top\vx_j + (\alpha_{l,2}+\beta_{l-1,2}\cdot\dot\alpha_{l,0})(t_i+\psi_i)(t_j+\psi_j) \\
	&+ (\alpha_{l,3}+\beta_{l-1,3}\cdot\dot\alpha_{l,0}+\beta_{l-1,1}\cdot\dot\alpha_{l,1})\left(\frac{1}{p}\beps_i^\top\beps_j\right)^2 +S_{ij} + \order{p^{-3/2}},
\end{align*}
so that it holds that 
\begin{align*}
	\mK_{ij}^{(l)} &= \beta_{l,1}\vx_i^\top\vx_j + \beta_{l,2}(t_i+\psi_i)(t_i+\psi_i) +\beta_{l,3}\left(\frac{1}{p}\beps_i^\top\beps_j\right)^2+S_{ij}+\order{p^{-3/2}}.
\end{align*}
with
\begin{align*}
	\beta_{l,1} &= \alpha_{l,1} + \beta_{l-1,1}\dot\alpha_{l,0},\\
	\beta_{l,2} &= \alpha_{l,2} + \beta_{l-1,2}\dot\alpha_{l,0},\\
	\beta_{l,3} &= \alpha_{l,3} + \beta_{l-1,3}\dot\alpha_{l,0}+\beta_{l-1,1}\dot\alpha_{l,1}.\\
\end{align*}
As $l\rightarrow \infty$, it holds that $\lim_{l\rightarrow\infty}\tau_l=\tau_*$,
$\lim_{l\rightarrow\infty}\alpha_{l,k}=\alpha_{*,k}$, and $\lim_{l\rightarrow\infty}\dot\alpha_{l,k}=\dot\alpha_{*,k}$, for $k=1,2,3$.
Therefore, for $l\rightarrow \infty$, one gets that
\begin{align*}
	\mK_{ij}^{*} = \beta_{*,1}\vx_i^\top\vx_j + \beta_{*,2}(t_i+\psi_i)(t_j+\psi_j) + \beta_{*,3}\left(\frac{1}{p}\beps_i^\top\beps_j\right)^2 + S_{ij} + \order{p^{-3/2}},
\end{align*}
where
\begin{align*}
	\beta_{*,1} &\equiv\lim_{l\rightarrow\infty}\beta_{l,1}= (1-\dot\alpha_{*,0})^{-1}\alpha_{*,1},\\
	\beta_{*,2} &\equiv\lim_{l\rightarrow\infty}\beta_{l,2}=  (1-\dot\alpha_{*,0})^{-1}\alpha_{*,2} ,\\
	\beta_{*,3} &\equiv\lim_{l\rightarrow\infty}\beta_{l,3}= (1-\dot\alpha_{*,0})^{-1}(\alpha_{*,3} + \beta_{*,1}\dot\alpha_{*,1}).\\
\end{align*}
\paragraph{Assembling in matrix form:} Using the fact that $\|\mM\|_2\leq n\max_{i,j}|\mM_{ij}|$ for $\mM\in\R^{n\times n}$ and $\{S_{ij}\}_{ij}=\mathcal{O_{\|\cdot\|}}(p^{-1/2})$~\cite{couillet2016kernel}, it holds that 
\begin{equation}
	\mK^* = \beta_{*,1}\mX^\top\mX + \mV\mD_*\mV^\top + (\kappa_*^2-\tau_0^2\beta_{*,1})\mI_n +\mathcal{O_{\|\cdot\|}}(p^{-\frac{1}{2}}),
\end{equation}
with 
\begin{equation}
	\mV = \left[\mJ/\sqrt{p},\boldsymbol{\psi}\right], \quad \mD_* = \left[\begin{array}{cc}
		\beta_{*,2}\vt\vt^\top +\beta_{*,3}\mT& \beta_{*,2}\vt  \\
		\beta_{*,2}\vt^\top & \beta_{*,2}
	\end{array}\right],
\end{equation}
and 
\begin{equation}
	\mT = \left\{\frac{1}{p}\tr \mC_a\mC_b\right\}_{a,=1}^{K}, \quad \vt = \left\{\frac{1}{\sqrt{p}}\tr \mC_a^\circ\right\}.
\end{equation}

\section{Proof and discussions of Examples~\ref{exm:tanh} and~\ref{exm:ReLU}} \label{sec:proof_of_exm}
Let $\ttau_0=\tau_0$ as defined in \eqref{eq:def_tau0} and $\ttau_1,\cdots\ttau_L\geq0$ be a sequence of non-negative scalars satisfying $\ttau_l = \sqrt{\E[\sigma_l^2(\ttau_{l-1}\xi)]}$, for $\xi\sim\gN(0,1)$ and $l\in\{1,\cdots,L\}$.
It follows from~\ref{thm:exCK}  that $\|\bsigma^{(l)}-\overline{\bsigma}^{(l)}\|=\order{n^{-1/2}}$, where 
    \begin{equation}
        \overline{\bsigma}^{(l)} = \talpha_{l,1}\mX^\top\mX + \mV\tmC_l\mV^\top + (\ttau_l^2-\tau_0^2\talpha_1)\mI_n,
    \end{equation}
   with $\mV \in\R^{n\times (K+1)}$ as defined in Theorem~\ref{thm:imCK}, 
    \begin{align*}
        \tmC_l &= \left[\begin{array}{cc}
        \talpha_{l,2}\vt\vt^\top +\talpha_{l,3}\mT& \talpha_{l,2}\vt  \\
        \talpha_{l,2}\vt^\top & \talpha_{l,2}
                 \end{array}\right]\in\R^{(K+1)\times (K+1)}.      
    \end{align*}
     and non-negative scalars $\talpha_{l,1}, \talpha_{l,2}, \talpha_{l,3}$ defined recursively as $\talpha_{0,1}=\talpha_{0,4}=1$, $\talpha_{0,2}=\talpha_{0,3}=0$, and
\begin{equation*}
\begin{split}
    \talpha_{l,1} &= \E[\sigma'_l(\ttau_{l-1}\xi)]^2\talpha_{l-1,1},\\ 
    \talpha_{l,2} &=\E[\sigma'_l(\ttau_{l-1}\xi)]^2\talpha_{l-1,2} +\frac{1}{4} \E[\sigma''_l(\ttau_{l-1}\xi)]^2\talpha_{l-1,4}^2,\\
    \talpha_{l,3} &= \E[\sigma'_l(\ttau_{l-1}\xi)]^2\talpha_{l-1,3} + \frac{1}{2}\E[\sigma''_l(\ttau_{l-1}\xi)]^2\talpha_{l-1,1}^2.
\end{split}
\end{equation*}
For a given {\sf Tanh-DEQ},  we first compute the four key parameters $\alpha_{*,1}, \alpha_{*,2},\alpha_{*,3}$ and $\gamma_*$ of the implicit DEQ according to ~\eqref{eq:alpha*} of \Cref{thm:imCK}. For the single-hidden-layer Hard-tanh explicit NN,
it can be easily checked that the corresponding CK matrix is determined by
\begin{equation*}
    \talpha_{1,1} = \E[\sigma'_l(\tau_{0}\xi)]^2,\quad
    \talpha_{1,2} = \talpha_{1,3} = 0.
\end{equation*}

For {\sf H-Tanh-ENN} with the activation $\sigma_{\operatorname{H-Tanh}}(x) \equiv ax \cdot 1_{-c\leq x\leq c}  + ac\cdot(1_{x\geq c} - 1_{ x\leq -c} )$ with $a>0$ and $c\geq0$,  $\talpha_{1,1}, \talpha_{1,1}, \talpha_{1,1}, \ttau_1$ can be represented as functions of the activation parameters by following results
\begin{align*}
    \E[(\sigma_{\operatorname{H-Tanh}}(\tau_0\xi))^2] &= \frac{1}{2}\left(c^2+a^2 + (c^2-a^2)e^{-2\tau_0^2}-c^2e^{-\tau_0^2}\right), \quad
    \E[(\sigma_{\operatorname{H-Tanh}}(\tau_0\xi))'] = a e^{-\tau_0^2/2},\\
    \E[(\sigma_{\operatorname{H-Tanh}}(\tau_0\xi))''] & = -ce^{-\tau_0^2/2},\quad
    \E[((\sigma_{\operatorname{H-Tanh}}(\tau_0\xi))^2)''] = 2e^{-2\tau_0^2}\left(a^2+c^2(e^{\tau_0^2}-1)\right).
\end{align*}
To match {\sf Tanh-DEQ}, we determine the activations of {\sf H-Tanh-ENN} by solving the system 
 \begin{equation*}
     \talpha_{1,1} = \alpha_{*,1},\quad \talpha_{1,2} = \alpha_{*,2},\quad \talpha_{1,3} = \alpha_{*,3},\quad \ttau_1=\gamma_*.
 \end{equation*}
For a given {\sf ReLU-DEQ},  we first compute the four key parameters $\alpha_{*,1}, \alpha_{*,2},\alpha_{*,3}$ and $\gamma_*$ of the implicit DEQ according to ~\eqref{eq:alpha*} of \Cref{thm:imCK}. For a two-hidden-layer explicit NN,
it can be easily checked that the corresponding CK matrix is determined by
\begin{align*}
	\widetilde\alpha_{2,1} &= \E[\sigma_2'(\widetilde\tau_1\xi)]^2\talpha_{1,1}\\
	\widetilde\alpha_{2,2}&= \E[\sigma_2'(\widetilde\tau_1\xi)]^2\widetilde\alpha_{1,2} +\frac{1}{4}\E[\sigma_2''(\widetilde\tau_{1}\xi)]^2\widetilde\alpha_{1,4}^2,\\
	\widetilde\alpha_{2,3} &= \E[\sigma_2'(\widetilde\tau_1\xi)]^2\widetilde\alpha_{1,3} +\frac{1}{2}\E[\sigma_2''(\widetilde\tau_{1}\xi)]^2\widetilde\alpha_{1,1}^2,
\end{align*}
and  
\begin{align*}
	\talpha_{1,1} = \E[\sigma_1'(\ttau_0\xi)]^2,\quad
	\widetilde\alpha_{1,2}=\frac{1}{4}\E[\sigma_{1}''(\widetilde\tau_0\xi)]^2,\quad
	\widetilde\alpha_{1,3}=\frac{1}{2}\E[\sigma_{1}''(\widetilde\tau_0\xi)]^2,\quad
	\widetilde\alpha_{1,4}=\frac{1}{2}\E[(\sigma_{1}^2(\tau_0\xi))''].
\end{align*}
For {\sf L-ReLU-ENN} with the activation $\sigma_{\operatorname{L-ReLU}}^{(l)}(x) \equiv \max(a_lx,b_lx)-\frac{a_l-b_l}{\sqrt{2 \pi } }\tau_{l-1}$, $\talpha_{2,1}, \talpha_{2,1}, \talpha_{2,1}, \ttau_1$ can be represented as functions of the activation parameters by following results
\begin{align*}
	\E[(\sigma_{\operatorname{L-ReLU}}^{(l)}(\tau_{l-1}\xi))^2] &= \frac{(a_l^2+b_l^2)(\pi-1)+2a_lb_l}{2\pi}\tau_{l-1}^2,\quad
	\E[(\sigma_{\operatorname{L-ReLU}}^{(l)}(\tau_{l-1}\xi))'] =\frac{a_l+b_l}{2},\\
	\E[(\sigma_{\operatorname{L-ReLU}}^{(l)}(\tau_{l-1}\xi))'']^2 &=\frac{a_l-b_l}{\sqrt{2}\pi\tau_{l-1}} ,  \quad \E[((\sigma_{\operatorname{L-ReLU}}^{(l)}(\tau_{l-1}\xi))^2)'']=\frac{(a_l^2+b_l^2)(\pi-1)+2a_lb_l}{\pi}.
\end{align*}
To match {\sf ReLU-DEQ}, we determine the activations of {\sf L-ReLU-ENN} by solving the system 
 \begin{equation*}
     \talpha_{2,1} = \alpha_{*,1},\quad \talpha_{2,2} = \alpha_{*,2},\quad \talpha_{2,3} = \alpha_{*,3},\quad \ttau_2=\gamma_*.
 \end{equation*}

\paragraph{On the numerical determination of $\sigma_{\operatorname{L-ReLU}}^{(l)}$ and $\sigma_{\operatorname{H-Tanh}}$.} The system of nonlinear equations mentioned above does not admit explicit solutions but can be efficiently solved using numerical methods, such as the least squares method (implemented through the {\sf optimize.minimize} function in the {\sf SciPy} library).

\section{Additional Experimental Results}
\label{sec:additional_expri}

\subsection{High dimensional equivalents of  Implicit-NTKs}~\label{sm: real-imp}
\begin{figure}
    \centering
         \begin{tikzpicture}[font=\large,spy using outlines, inner sep=1.2]
      \pgfplotsset{every major grid/.style={style=densely dashed}}
      \begin{axis}[
        height=.4\linewidth,
        width=.45\linewidth,
        ymin=0,
        ymax=0.3,
        xtick={0,300,600,900,1200},
        xticklabels = {0,300,600,900,1\,200},
        ytick = {0.05, 0.10, 0.15, 0.20, 0.25},
        yticklabels = {0.05, 0.10, 0.15, 0.20, 0.25},
        grid=major,
        scaled ticks=true,
        xlabel={ $p$ },
        ylabel={Relative error},
        legend style = {at={(0.98,0.50)}, anchor=south east, font=\small}
        ]
        \addplot+[
    BLUE, mark=*,line width=0.75pt,
    error bars/.cd, 
      y fixed,
      y dir=both, 
      y explicit
  ]  table[x=X, y=Y, y error plus=ErrorMax, y error minus=ErrorMin] {
          X    Y                      ErrorMax            ErrorMin 
          40   0.25801451163940115   0.033675003560810586   0.012102267094539477
  80   0.10669595601600733   0.0020641329821728506   0.002915385280972822
  160   0.058977334097823735   0.01775054088513595   0.021714204619843136
  320   0.03712246106794893   0.010744561896815295   0.0002605158145054338
  640   0.032624484947360255   0.00706170521138702   0.015861163834706277
  1280   0.025926297911583285   0.007336750439122211   0.008945986234725076
        };
        \addlegendentry{ {\sf ReLU } };
        \addplot+[
    RED, mark=+,line width=0.75pt,
    error bars/.cd, 
      y fixed,
      y dir=both, 
      y explicit
  ]  table[x=X, y=Y, y error plus=ErrorMax, y error minus=ErrorMin] {
          X Y  ErrorMax ErrorMin
          40   0.22775880550109923   0.02611175679063997   0.018367036811758114
  80   0.12182534665612896   0.0024361154014797197   0.0019519880692893056
  160   0.07846557819610805   0.0025029836631263944   0.004562823793107013
  320   0.05850928046101026   0.020405391450528537   0.006629303889039899
  640   0.05635250963074641   0.003831994708707602   0.001886692144944425
  1280   0.0537252407632456   0.00018592745740063463   0.00013218691151002725
        };
        \addlegendentry{ {\sf Tanh } };
        \addplot+[
    PURPLE, mark=triangle*,line width=0.75pt,
    error bars/.cd, 
      y fixed,
      y dir=both, 
      y explicit
  ]  table[x=X, y=Y, y error plus=ErrorMax, y error minus=ErrorMin] {
          X Y  ErrorMax ErrorMin
          40   0.2438274803831467   0.01914315402840536   0.0270606635137619
  80   0.10243503293196059   0.0034784957539763526   0.0030408460743492497
  160   0.057144331226263015   0.00958467442768559   0.00724815119038768
  320   0.0357070687773711   0.011241175497991527   0.003110140933883736
  640   0.024770002615739715   0.009713156249781982   0.005761978325517871
  1280   0.022899866994560415   0.002878286871682228   0.0030953944344635193
        };
        \addlegendentry{ {\sf Swish } };
        \addplot+[
    GREEN, mark=diamond*,line width=0.75pt,
    error bars/.cd, 
      y fixed,
      y dir=both, 
      y explicit
  ]  table[x=X, y=Y, y error plus=ErrorMax, y error minus=ErrorMin] {
          X Y  ErrorMax ErrorMin
          40   0.26096276263643575   0.01001255227914362   0.00985669699662369
  80   0.11253593240202883   0.004988001931285914   0.01801928580354153
  160   0.06308739845480808   0.010491961543749637   0.0011092531243079625
  320   0.036458495327150085   0.010709043727843894   0.0023431814171495707
  640   0.029113261228233153   0.0034865966259069853   0.0001527306901040939
  1280   0.024825914357279537   0.001828410347686781   0.001028048460086777
        };
        \addlegendentry{ {\sf L-ReLU } };
      \end{axis}
    \end{tikzpicture}
        
  \medskip
  \caption{Evolution of relative  spectral norm  error $\|\mK^*-\overline{\mK}\| / \|\mK^*\|$  \emph{w.r.t.}\@ sample size $n$, for DEQs in \Cref{def:deq} with different activations and $\sigma_a^2=0.2$, on two-class GMM, $p/n = 0.8$, $\vmu_a=[\mathbf{0}_{8(a-1)};8;\mathbf{0}_{p-8a+7}]$, and $\mC_a=(1+8(a-1)/\sqrt{p})\mI_p, a\in \{1,2\}$.
 Implicit-NTK matrices $\mK^*$ defined in~\eqref{eq:imntk} are taken with expectation estimated from DEQs with random $\mA$ and $\mB$ of width $m=2^{12}$. The asymptotic equivalent matrices $\overline{\mK}$ are obtained by Theorem~\ref{thm:imNTK}.}
  \label{fig:equivntk}
  \end{figure}
\Cref{fig:equivntk} compares the difference between Implicit-NTKs $\mK^*$ and their high-dimensional approximation $\overline{\mK}$ given in \Cref{thm:imNTK}, on two-class Gaussian mixture data, on DEQs as \Cref{def:deq} with four commonly used activations: ReLU, Tanh, Swish, and Leaky-ReLU (L-ReLU). 
We observe from \Cref{fig:equivntk} that, for different activations, as $n,p$ increase, the relative errors consistently and significantly decrease, as in line with our Theorem~\ref{thm:imNTK}.

\subsection{Visualization results of the spectrum of Implicit-CKs and those Implicit-NTKs}~\label{sm:realCKNTK} 
In Figure~\ref{fig:realCK} and Figure~\ref{fig:realNTK}, we provide visualization results of the spectral densities of Implicit-CKs and Implicit-NTKs, respectively, along with their corresponding high-dimensional approximations. We observe from Figure~\ref{fig:realCK} and Figure~\ref{fig:realNTK} that the proposed theoretical results, despite derived here for GMM data \emph{and} in the limit of $n,p \to \infty$, provide extremely accurate prediction of the Implicit-CK eigenspectral behavior (\romannumeral1) for not-so-large $n,p$ \emph{and} (\romannumeral2) possibly surprisingly, also on realistic MNIST data.
\input{fig/fig_realCK}
\input{fig/fig_realNTK}

\subsection{Visualization results of the spectrum of Implicit-CKs and equivalent Explicit-CKs}~\label{sm:realckmatching}
\input{fig/fig_realCKmatching}
In Figure~\ref{fig:realCKmatching}, we compare the spectral densities of Implicit-CK matrices of given DEQs with those of Explicit-CK matrices of the corresponding ``equivalent'' shallow explicit NNs by following Examples~\ref{exm:tanh} and~\ref{exm:ReLU}. We observe that the CK matrices of {\sf ENN}s are close to those of the corresponding{\sf DEQ}s. This observation is consistent on GMM data and realistic MNIST data. 
We conjecture that this is due to a high-dimensional \emph{universal} phenomenon and that our results (on both CK and NTK matrices) hold more generally beyond the GMM setting, say, for data drawn from the family of concentrated random vectors \citep{ledoux2005concentration,louart2019concentration}. 

\subsection{Adam Results}~\label{sm:adam_results} 
We present the classification results of implicit DEQs and explicit models trained with the Adam optimizer in Figure~\ref{fig:matching_adam}.   Each model is trained with the Adam optimizer, using initial learning rates  of $10^{-2}$ for MNIST and Fashion-MNIST, and $10^{-3}$ for CIFAR-10.  The remaining experimental settings mirror those of the SGD experiment depicted in Figure~\ref{fig:matching_sgd}. The results obtained with  the Adam optimizer  are similar to those achieved with  SGD.
\input{fig/NN_matching_adam}

\subsection{Time cost comparison}~\label{sm:timecost}
We compare the time costs of the inference and training of {\sf DEQ}s and the corresponding ``equivalent'' {\sf ENN}s. As shown in Tables~\ref{tab:infer} and~\ref{tab:train}, the inference time cost of a {\sf DEQ} is about $2-3\times$ that of an {\sf ENN} with the same dimension. This is due to the fact that it  takes numerous iterations for a {\sf DEQ} to reach the iteration error threshold. Additionally, we observe that {\sf ENN}s have a remarkable advantage over {\sf DEQ}s in terms of training speed.
\begin{table}
\centering
\normalsize
\begin{tabular}{cccccccccc}
\toprule
\multicolumn{2}{c}{Dimension}            & 32     & 64     & 128    & 256    & 512    & 1024   & 2048   & 4096    \\ \midrule
\multirow{4}{*}{MINIST}         & ReLU-DEQ   & 0.26 & 0.26 & 0.27 & 0.28 & 0.27 & 0.29 & 0.30 & 0.32  \\
                                & L-ReLU-ENN & 0.09 & 0.11 & 0.12 & 0.15 & 0.14 & 0.11 & 0.10 & 0.16   \\ \cmidrule{2-10} 
                                & Tanh-DEQ   & 0.24 & 0.24 & 0.27 & 0.26 & 0.28 & 0.26 & 0.28 & 0.29  \\
                                & H-Tanh-ENN & 0.12 & 0.12 & 0.11 & 0.11 & 0.14 & 0.11 & 0.10 & 0.10  \\ \midrule
\multirow{4}{*}{Fashion MINIST} & ReLU-DEQ   & 0.26 & 0.27 & 0.26 & 0.26 & 0.32 & 0.32 & 0.31 & 0.32  \\
                                & L-ReLU-ENN & 0.17 & 0.13 & 0.14 & 0.11 & 0.12 & 0.13 & 0.12 & 0.14   \\ \cmidrule{2-10} 
                                & Tanh-DEQ   & 0.22 & 0.24 & 0.25 & 0.26 & 0.28 & 0.27 & 0.28 & 0.30   \\
                                & H-Tanh-ENN & 0.13 & 0.13 & 0.12 & 0.12 & 0.11 & 0.12 & 0.10 & 0.12  \\ \midrule
\multirow{4}{*}{CIFAR 10}       & ReLU-DEQ   & 0.23 & 0.22 & 0.23 & 0.24 & 0.27 & 0.28 & 0.28 & 0.31    \\
                                & L-ReLU-ENN & 0.09 & 0.08 & 0.09 & 0.09 & 0.09 & 0.09 & 0.10 & 0.10   \\ \cmidrule{2-10}  
                                & Tanh-DEQ   & 0.20 & 0.20 & 0.21 & 0.22 & 0.23 & 0.23 & 0.26 & 0.27   \\
                                & H-Tanh-ENN & 0.09 & 0.08 & 0.09 & 0.09 & 0.09 & 0.08 & 0.08 & 0.09   \\ \bottomrule
\end{tabular}

\caption{Comparison of the inference time for a single input image between DEQs  and explicit NNs across different datasets. The inference time is recorded on a machine with Intel(R) Xeon(R) Gold 6138 CPU @ 2.00GHz with a single 3090 GPU. }
\label{tab:infer}
\end{table}
\begin{table}
\resizebox{\textwidth}{!}{
\begin{tabular}{cccccccccc}
\toprule
\multicolumn{2}{c}{Dimension}            & 32     & 64     & 128    & 256    & 512    & 1024   & 2048   & 4096    \\ \midrule
\multirow{4}{*}{MINIST}         & ReLU-DEQ    & 68.76  & 72.72  & 82.00     & 95.52 & 97.38  & 93.96  & 110.04 & 189.54  \\
                                & L-ReLU-ENN & 3.045  & 2.67   & 3.52  & 3.75  & 4.12  & 3.65  & 3.09  & 3.97   \\ \cmidrule{2-10} 
                                & Tanh-DEQ    & 65.34 & 83.94  & 84.72  & 84.78  & 85.14  & 94.14  & 306.00    & 322.56  \\
                                & H-Tanh-ENN & 2.76  & 5.32  & 2.80  & 2.65  & 2.96  & 2.69  & 6.59  & 2.64   \\ \midrule
\multirow{4}{*}{Fashion MINIST} & ReLU-DEQ    & 58.46  & 81.72  & 72.78  & 73.38  & 141.72 & 156.42 & 169.74 & 246.42  \\
                                & L-ReLU-ENN & 3.53   & 3.17  & 3.07  & 2.70  & 2.84  & 3.06  & 2.71  & 3.11   \\ \cmidrule{2-10} 
                                & Tanh-DEQ    & 70.80   & 80.28  & 81.48  & 79.38  & 91.50   & 96.00     & 109.98 & 113.04  \\
                                & H-Tanh-ENN & 2.98  & 2.78  & 2.76  & 2.19  & 2.71  & 3.30  & 2.19  & 2.47   \\ \midrule
\multirow{4}{*}{CIFAR 10}       & ReLU-DEQ    & 225.00    & 253.26 & 302.16 & 172.68 & 870.00    & 1167.60 & 1208.40 & 1204.20  \\
                                & L-ReLU-ENN & 2.47   & 3.21  & 2.28  & 3.05  & 4.65  & 5.34   & 4.71  & 4.53   \\ \cmidrule{2-10}  
                                & Tanh-DEQ    & 68.60   & 79.26  & 89.22  & 92.98  & 108.00    & 192.46 & 254.46 & 307.56 \\
                                & H-Tanh-ENN & 2.24  & 2.90  & 2.91  & 3.06  & 2.99  & 2.91  & 3.36  & 7.17   \\ \bottomrule
\end{tabular}
}
\caption{Comparison of the training time for one epoch (batch size $128$) between DEQs 
and explicit NNs across different datasets. The running time is recorded on a machine with Intel(R) Xeon(R) Gold 6138 CPU @ 2.00GHz with a single 3090 GPU. }
\label{tab:train}
\end{table}

\end{document}